\documentclass[10pt,journal,compsoc]{IEEEtran}

\usepackage{xspace}
\usepackage[table]{xcolor}

\usepackage{amsmath}
\usepackage{amssymb}
\usepackage{amsthm}
\usepackage{bm}
\usepackage{graphicx}
\usepackage{subcaption}
\usepackage{booktabs}
\usepackage{multirow}
\usepackage{makecell}
\usepackage[ruled,linesnumbered]{algorithm2e}
\usepackage{tikz}
\usepackage{algpseudocode}

\usepackage{cite}
\usepackage[switch]{lineno}
\usepackage[colorlinks,allcolors=cvprblue,breaklinks]{hyperref}
\usepackage{url}

\definecolor{cvprblue}{rgb}{0.21,0.49,0.74}
\definecolor{tabhighlight}{HTML}{e5e5e5}
\definecolor{video_quality}{HTML}{F0D695}
\definecolor{user_study}{HTML}{82ACD1}
\definecolor{score_average}{HTML}{98B567}

\newcommand{\method}{\texttt{VDOT}\xspace}
\newcommand{\methodbf}{\textbf{\texttt{VDOT}}\xspace}
\newcommand{\methodpp}{\texttt{VDOT\small{++}}\xspace}
\newcommand{\methodppbf}{\textbf{\texttt{VDOT\small{++}}}\xspace}
\newcommand{\methodconf}{\texttt{VDOT}\xspace}          
\newcommand{\methodbench}{\texttt{UVCBench}\xspace}

\newcommand{\real}{\mathrm{real}}
\newcommand{\fake}{\mathrm{fake}}
\newcommand{\pfake}{p_{\fake}}
\newcommand{\preal}{p_{\real}}
\newcommand{\sg}{\mathrm{sg}}
\newcommand{\Wdist}{\mathcal{D}_{\mathrm{OT}}}
\newcommand{\Weps}{\mathcal{D}_{\mathrm{EOT}}^{\epsilon}}
\newcommand{\Tplan}{\bm{T}}
\newcommand{\Topt}{\bm{T}^{*}}
\newcommand{\Cmat}{\bm{C}}                              
\newcommand{\softplus}{\mathrm{sp}}                     

\newtheorem{proposition}{Proposition}

\theoremstyle{remark}

\begin{document}

\title{VDOT++: Unified Few-Step Video Generation via Unbalanced Optimal Transport Distillation}

\author{
  Yutong Wang,
  Xingtong Ge,
  Enhuai Liu,
  Yunke Wang,
  Tianfan Xue,
  Yu~Qiao,~\IEEEmembership{Senior Member,~IEEE},
  Yaohui Wang, 
  Xinyuan Chen,
  and Chang~Xu,~\IEEEmembership{Senior Member,~IEEE}
  \IEEEcompsocitemizethanks{
    \IEEEcompsocthanksitem Yutong Wang, Enhuai Liu, Yunke Wang and Chang Xu are with the University of Sydney.
    \IEEEcompsocthanksitem Tianfan Xue, Yu Qiao, Yaohui Wang and Xinyuan Chen are with Shanghai AI Laboratory.
    \IEEEcompsocthanksitem Xingtong Ge is with the Hong Kong University of Science and Technology. Tianfan Xue is with the Chinese University of Hong Kong.
    \IEEEcompsocthanksitem Corresponding authors: Xinyuan Chen, Chang Xu.
  }
}

\markboth{Preprint}%
{Wang \MakeLowercase{\textit{et al.}}: VDOT++: Unified Few-Step Video Generation via Unbalanced Optimal Transport Distillation}

\IEEEtitleabstractindextext{%
  \begin{abstract}
Video creation spans text-to-video (T2V), image-to-video (I2V), and
condition-based generation, yet modern video diffusion models remain costly
because they repeatedly evaluate large backbones during sampling. Distribution
matching distillation (DMD)
reduces this cost, but its reverse Kullback--Leibler (KL) objective can provide
unstable or incomplete guidance when the student and teacher distributions
have limited overlap. A preliminary version of this work, \methodconf\
(CVPR~2026), addressed this issue by adding optimal transport distillation
(OTD), whose explicit coupling supplies geometric directions for
condition-based generation. Balanced OTD, however, performs full-mass matching
between the spatial tokens of each corresponding student--teacher frame pair.
This assumption weakens for T2V and I2V, where one condition admits many valid
outputs and spatial content need not align across different realizations. We
present \methodpp, a unified distillation framework that applies the same
training recipe separately to task-specific generators for the three task
families. It makes OTD robust to output diversity through an asymmetric
unbalanced formulation that allows unreliable student tokens to carry less
mass while maintaining coverage of the teacher tokens. An $\ell_1$
ground cost further replaces mean-based aggregation
with a more mode-preserving, coordinate-wise weighted median that limits the
influence of distant transport targets. The two changes respectively determine
whom to match and how the selected targets should be aggregated. We additionally
combine
distribution matching and adversarial refinement through sequential backward
passes, and exploit the decoupled score networks for cross-scale distillation,
where larger auxiliary score networks improve a compact generator. A scalable
multi-task data pipeline and the \methodbench\ benchmark support evaluation
across heterogeneous conditions. Experiments on \methodbench, VBench,
VBench-I2V, and the VACE
benchmark show that the resulting four-step generators are competitive with
many-step teachers and strong few-step baselines across all three task families.

\end{abstract}

}

\maketitle
\IEEEdisplaynontitleabstractindextext

\section{Introduction}
\label{sec:intro}

\IEEEPARstart{V}{ideo} creation spans generation and editing from text,
reference images, source videos, and spatial masks. Unified models such as
VACE~\cite{jiang2025vace} and UNIC~\cite{ye2025unic} support these conditions
within a shared architecture, but their large diffusion backbones still
require many denoising evaluations. Few-step distillation can reduce this
cost, yet its supervision must remain reliable across tasks with very
different degrees of output ambiguity.

Distribution matching distillation (DMD)~\cite{yin2024one,yin2024improved}
uses an approximate reverse-KL gradient to align a few-step student with a
many-step teacher. When their distributions overlap poorly, this signal may
miss teacher modes or become unstable on unsupported student samples. Our
conference method, \methodconf~\cite{wang2026vdot}, supplements DMD with
optimal transport distillation (OTD), whose explicit coupling supplies
geometric directions for condition-based generation.

Balanced OTD, however, solves a full-mass spatial-token matching problem
independently within each corresponding student--teacher frame pair. Dense
conditions such as depth or pose make these within-frame correspondences
reliable, whereas T2V and single-frame I2V admit many valid appearances and
motions. In the latter regime, hard marginals force every student token to
carry prescribed mass even when no reliable teacher match exists
(\emph{whom to match}), while the squared-$\ell_2$ cost induces mean-based
aggregation over potentially incompatible teacher tokens
(\emph{how to aggregate}). The challenge is to retain geometric guidance
without treating diversity among valid outputs as error.

\methodpp\ addresses the two failure modes with asymmetric unbalanced OT and
an $\ell_1$ ground cost. A relaxed student marginal downweights unreliable
student tokens, while a strongly constrained teacher marginal maintains
coverage of the teacher tokens. The $\ell_1$ cost replaces mean-based
aggregation with a coordinate-wise weighted median that is less sensitive to
distant targets. We further combine distribution matching and adversarial
refinement through memory-efficient sequential backward passes, and use larger
auxiliary score networks during distillation to strengthen a compact generator.
Our evaluation covers single and composite condition-based tasks from
\methodbench, as well as established benchmarks for T2V and I2V.

Experiments reveal a consistent conditioning-dependent pattern: the relaxed
objective behaves similarly to balanced OTD when dense inputs establish
reliable correspondences, but better preserves valid motion and detail as the
admissible output space widens. Controlled ablations isolate matching, cost,
relaxation direction, adversarial refinement, and score-network scale.

A preliminary version of this work appeared at CVPR~2026~\cite{wang2026vdot}.
Compared with the conference version, this article makes three main
extensions:
\begin{itemize}
  \item We extend \methodconf\ from condition-based video creation to T2V and
  I2V, establishing a unified few-step distillation framework across all three
  task families. When the student and teacher distributions overlap poorly,
  OTD complements DMD's potentially unreliable score-difference gradients by
  constructing explicit token-level transport plans between their outputs.
  \item We identify forced correspondence and mean-based aggregation as two
  distinct failures of balanced, squared-cost OTD under output diversity, and
  replace it with asymmetric unbalanced transport and an $\ell_1$ ground cost.
  The former addresses \emph{whom to match} by relaxing unreliable
  student-token mass while preserving teacher-token coverage; the latter
  addresses \emph{how to aggregate} by limiting averaging across incompatible
  targets. Controlled ablations validate their complementary roles across
  conditioning strengths.
  \item We introduce a scalable training design that combines distribution
  matching and adversarial refinement through sequential backward passes,
  accumulating a joint update within the memory budget of large video
  backbones. We further exploit the decoupling between the generator and score
  networks for cross-scale distillation, using larger auxiliary score networks
  to provide stronger supervision and improve the generation quality of a
  compact generator.
\end{itemize}

\section{Related Work}
\label{sec:related}

\subsection{Unified Video Generation and Editing}
Unified generative models encode heterogeneous controls through a shared
conditioning interface. Image methods attach reusable control branches or
unify generation and editing within one architecture
~\cite{zhang2023adding,mou2024t2i,xiao2025omnigen,tan2025ominicontrol,wei2026univideo,wang2023videocomposer}.
Recent video systems extend this idea across multiple creation tasks.
VACE~\cite{jiang2025vace} represents text, images, videos, and masks as aligned
context frames and masks, whereas UNIC~\cite{ye2025unic} maps heterogeneous
inputs into shared token categories. Because these methods build on large
diffusion backbones~\cite{kong2024hunyuanvideo,wan2025}, they retain the
cost of iterative sampling. Our focus is complementary: rather than designing
a unified backbone, we develop a unified distillation framework that accelerates
pretrained models for VACE-style controllable creation, T2V, and I2V.

\subsection{Few-Step Diffusion Distillation}
Video diffusion acceleration has been explored through model
pruning~\cite{karnewar2025neodragon,wang2026pare,su2024f3}, quantization~\cite{li2025dvdquant,li2026deltaquant,wang2026pulsequant}, sparse attention~\cite{xi2025sparse,yang2026sparse,li2026vc,liu2026losa}, feature caching~\cite{liu2025timestep,liu2026d2cache,liu2025reusing},
and efficient sampling~\cite{yin2024one,yin2024improved,ge2026saltpp}. 
Among sampling acceleration approaches, few-step distillation
trains a student to approximate many-step generation with
a small number of denoising steps. 
Few-step diffusion distillation spans three relevant families.
\emph{Consistency distillation} compresses the teacher trajectory or enforces consistent predictions along it~\cite{salimans2022progressive,song2023consistency}. 
rCM further scales
continuous-time consistency training to large image and video models by
introducing score regularization~\cite{zheng2025large}.
\emph{Distribution matching} instead aligns the student and teacher output
distributions. 
DMD and DMD2 estimate a reverse-KL gradient using teacher and
fake score networks~\cite{yin2024one,yin2024improved}; CausVid,
Self-Forcing, and Salt adapt distribution-level training to autoregressive video
generation~\cite{yin2025causvid,huang2025selfforcing,ge2026salt}, while TDM matches few-step transition trajectories~\cite{luo2025tdm}.
\emph{Adversarial distillation} uses discriminators and real or
teacher-synthesized data to restore perceptual detail, as in ADD, APT, ADP, and AccVideo~\cite{sauer2024adversarial,lin2025apt,lu2025adversarial,
zhang2025accvideo}. Among these objectives, DMD provides distribution-level supervision through score differences, but this signal can become unreliable when student samples lie far from the teacher distribution. Our conference method, \methodconf~\cite{wang2026vdot}, complements DMD with a token-level OT coupling, providing geometric guidance even under poor distribution overlap. The present work further revisits its marginal constraints and ground cost for tasks with different degrees of output ambiguity.

\subsection{Optimal Transport for Distribution Alignment}
Optimal transport (OT) compares distributions through both a ground cost and
an explicit transport plan~\cite{peyre2019computational}, with entropic
regularization enabling efficient Sinkhorn optimization
~\cite{cuturi2013sinkhorn}. Transport geometry has supported generative
modeling~\cite{tolstikhin2017wasserstein,arjovsky2017wasserstein,
tartavel2016wasserstein,wang2024inverse}, assignment~\cite{ge2021yolox,luo2022weakly,wang2023self}, and one-step image
editing~\cite{lu2026chordedit}. Balanced OT requires all mass to participate.
Partial OT can leave costly mass unmatched~\cite{figalli2010optimal}, whereas
unbalanced OT replaces hard marginal equalities with divergence penalties
~\cite{chizat2018unbalanced,choi2023uotm}. Although UOT permits distinct
penalties for the two marginals, generic distribution alignment does not
determine which side should be relaxed more. Distillation supplies this
directionality: the evolving student may contain unreliable tokens, whereas
the fixed teacher remains the coverage target. This asymmetry motivates
different marginal penalties, while an $\ell_1$ ground cost limits averaging
over incompatible teacher tokens.

\section{Preliminaries}
\label{sec:prelim}

We review distribution matching distillation (DMD) and the optimal transport
(OT) formulations on which our method builds. Throughout, $\bm{c}$ denotes
the task-specific condition: text for T2V, text and an initial image for I2V,
or, following VACE~\cite{jiang2025vace}, text, context frames, and aligned
masks for controllable creation. 
These tasks range from text-only generation with broad output ambiguity to
densely conditioned generation with reliable spatial correspondence between
paired frames.

\subsection{Distribution Matching Distillation}
\label{sec:prelim_dmd}
DMD~\cite{yin2024one,yin2024improved} distills a pretrained diffusion model
into a few-step generator $G_\theta$. A frozen teacher score network $F_\psi$
represents the target distribution $\preal$, while a trainable fake score
network $F_\phi$ tracks the evolving student distribution $\pfake$. DMD
minimizes their reverse Kullback--Leibler (KL) divergence:
\begin{equation}
  \mathbb{D}_{\mathrm{KL}}(\pfake\,\|\,\preal)
  = \int \pfake(\bm{x})\log\frac{\pfake(\bm{x})}{\preal(\bm{x})}\,d\bm{x}.
  \label{eq:kl}
\end{equation}
We suppress the shared condition $\bm{c}$ in the densities for clarity.
Let $\hat{\bm{x}}_0=G_\theta(\bm{z},\bm{c})$ be a generated video latent and
$\bm{x}_t=\alpha_t\hat{\bm{x}}_0+\sigma_t\bm{\xi}$, with
$\bm{\xi}\sim\mathcal{N}(\bm{0},\bm{I})$, its noised version at diffusion
time $t$. Denote the corresponding noised distributions by
$p_{\fake,t}$ and $p_{\real,t}$, with scores
$s_y(\bm{x}_t,t)=\nabla_{\bm{x}_t}\log p_{y,t}(\bm{x}_t)$ for
$y\in\{\real,\fake\}$. DMD approximates the KL gradient using the weighted
score difference
\begin{equation}
  \nabla_{\mathrm{KL}}(\bm{x}_t,t)
  = w_t\alpha_t\big(
      s_{\fake}(\bm{x}_t,t)-s_{\real}(\bm{x}_t,t)\big),
  \label{eq:dmd_grad}
\end{equation}
where $w_t$ is a time-dependent scalar that normalizes the gradient magnitude
across noise levels. This pseudo-gradient is injected into the generator
through the regression loss
\begin{equation}
  \mathcal{L}_{\mathrm{DMD}}(\theta)
  = \mathbb{E}_{\bm{z},\bm{c},t,\bm{\xi}}\!\left[
    \big\| \hat{\bm{x}}_0 - \sg\big(\hat{\bm{x}}_0
      - \nabla_{\mathrm{KL}}(\bm{x}_t,t)\big)\big\|_2^2
  \right],
  \label{eq:dmd}
\end{equation}
where $\sg(\cdot)$ denotes stop-gradient.

Reverse KL is mode-seeking, which is especially problematic early in
few-step distillation when the student and teacher distributions have little
overlap~\cite{wang2026vdot,lu2025adversarial}. If
$\pfake(\bm{x})\!\to\!0$ while $\preal(\bm{x})>0$, the local contribution
$\pfake\log(\pfake/\preal)$ tends to zero. Teacher-supported regions that the
student has not yet reached therefore exert vanishing pressure, yielding the
\emph{zero-forcing} behavior illustrated in Fig.~\ref{fig:reverse_kl_ot}. In
the opposite case, $\pfake(\bm{x})>0$ while $\preal(\bm{x})\!\to\!0$, the
local contribution diverges and can produce excessively large, unstable
gradients.
Thus, reverse KL alone provides weak coverage of missing teacher modes while
penalizing unsupported student samples extremely sharply. This motivates a
geometric constraint that explicitly couples the two distributions.

\begin{figure}[t]
  \centering
  \includegraphics[width=\columnwidth]{figures/reverse_kl_ot_illustration.png}
  \caption{\textbf{Reverse-KL failure modes and the role of OT.} When the
  student and teacher distributions are initially separated, reverse KL can
  neglect uncovered teacher modes or produce unstable gradients on
  unsupported student samples. An OT coupling supplies explicit transport
  directions that guide the student toward the teacher support. Adapted from
  the conference version~\cite{wang2026vdot}.}
  \label{fig:reverse_kl_ot}
\end{figure}

\subsection{Optimal Transport Discrepancies}
\label{sec:prelim_ot}
Given sample matrices $\bm{A}\in\mathbb{R}^{I\times D}$ and
$\bm{B}\in\mathbb{R}^{J\times D}$ representing two empirical distributions,
their discrete OT discrepancy~\cite{peyre2019computational} is
\begin{equation}
  \Wdist(\bm{A},\bm{B})
  = \min_{\Tplan\in\Pi(\bm{u},\bm{\mu})}
    \langle \Cmat,\Tplan \rangle,
  \label{eq:w2}
\end{equation}
where $C_{ij}=d(\bm{a}_i,\bm{b}_j)$ is the ground cost and
$\Pi(\bm{u},\bm{\mu})=\{\Tplan\ge0\mid
\Tplan\bm{1}_J=\bm{u},\,\Tplan^{\!\top}\bm{1}_I=\bm{\mu}\}$ is the set of
plans with prescribed marginals, typically chosen to be uniform. The optimal
plan $\Topt$ specifies both which samples should be coupled and how much mass
each pair should exchange. The conventional squared-$\ell_2$ cost is
$d(\bm{a},\bm{b})=\tfrac12\|\bm{a}-\bm{b}\|_2^2$.

For efficient, differentiable optimization, we use the entropic OT (EOT)
discrepancy. Define the negative-entropy functional
$\mathcal{E}(\Tplan)=\sum_{i,j}T_{ij}(\log T_{ij}-1)$, with
$0\log 0=0$. 
The problem can be solved in $\mathcal{O}(IJ)$ per iteration with the Sinkhorn
algorithm~\cite{cuturi2013sinkhorn}. The regularization strength $\epsilon$
controls the diffuseness of the resulting plan.
\begin{equation}
  \Weps(\bm{A},\bm{B})
  = \min_{\Tplan\in\Pi(\bm{u},\bm{\mu})}
    \langle \Cmat,\Tplan \rangle
    + \epsilon\,\mathcal{E}(\Tplan),
  \label{eq:eot}
\end{equation}

Both Eqs.~\eqref{eq:w2} and~\eqref{eq:eot} are \emph{balanced}: their hard
marginal constraints require all mass on both sides to participate. Partial
OT~\cite{figalli2010optimal} instead transports only a prescribed fraction
$\kappa_{\mathrm{POT}}\in(0,1]$ of the mass and leaves costly
correspondences unmatched.
Unbalanced OT~\cite{chizat2018unbalanced} achieves a related relaxation by
replacing the marginal equalities with divergence penalties, so that mass may
deviate from the prescribed marginals where matching is unreliable. Balanced
OT is recovered when all mass participates in partial OT
($\kappa_{\mathrm{POT}}=1$) or
when the unbalanced marginal penalties become arbitrarily strong.
Together, these formulations separate the coupling constraints from the
choice of ground cost. Section~\ref{sec:method} specializes both components
for OTD across different conditioning strengths.

\begin{figure*}[t]
  \centering
  \includegraphics[width=1.0\textwidth]{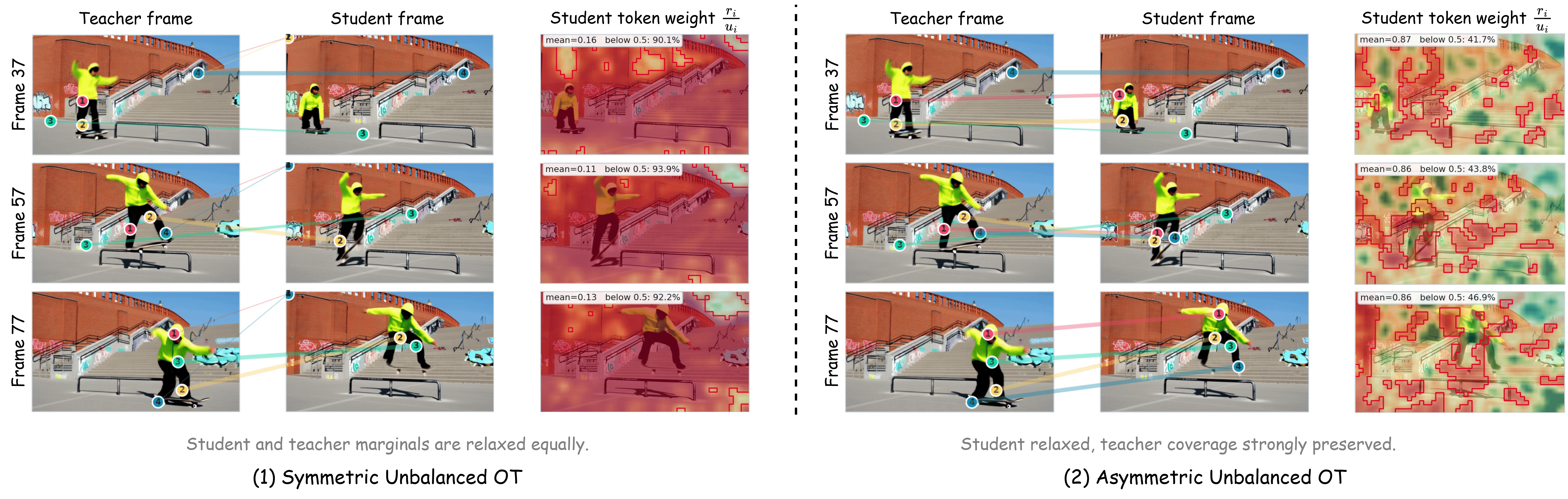}
  \caption{\textbf{Symmetric versus asymmetric unbalanced OT.}
  Under balanced OT, the hard student marginal gives $r_i/u_i=1$ for every
  token, producing a uniformly yellow map that cannot distinguish reliable
  from unreliable student tokens. With the frames, latent features, and
  ground costs fixed, symmetric UOT relaxes both marginals using
  $(\tau_s,\tau_t)=(0.1,0.1)$ (left), whereas our asymmetric variant uses
  $(\tau_s,\tau_t)=(0.1,10)$ (right) to relax the student side while
  preserving teacher coverage. Colored anchors and lines connect
  selected teacher tokens to their mass-weighted student matches under
  $\Topt$. The heat maps show relative student-token mass $r_i/u_i$, where
  $r_i=\sum_jT^*_{ij}$: red, yellow, and green indicate suppressed,
  near-uniform, and emphasized mass, and red contours mark $r_i/u_i<0.5$.
  Asymmetric relaxation preserves teacher coverage while selectively
  downweighting unreliable student regions.}
  \label{fig:teaser}
\end{figure*}

\section{Methodology}
\label{sec:method}

\subsection{Overview}
\methodpp\ distills a pretrained many-step video diffusion model into a
few-step generator using a shared training objective across task-specific
conditioning interfaces. We instantiate it with
VACE-Wan2.1-14B~\cite{jiang2025vace} for condition-based generation and
Wan2.1-14B~\cite{wan2025} for T2V and I2V. As
in DMD, a frozen teacher score network $F_\psi$ and a trainable fake-score
network $F_\phi$ guide the generator $G_\theta$. The conference version
introduced balanced OTD to supplement the reverse-KL signal with explicit
geometric correspondences~\cite{wang2026vdot}. We extend that formulation in
three respects: partial matching avoids forcing unreliable correspondences,
asymmetric marginal penalties preserve the distinct roles of student and
teacher, and an $\ell_1$ ground cost makes aggregation robust to diverse
targets. A lightweight adversarial head further corrects local appearance
errors. These losses are accumulated through sequential backward passes and
applied in a single generator update. Figure~\ref{fig:framework} summarizes
the framework; the remainder of this section develops each component and the
resulting training procedure.

\begin{figure*}[t]
  \centering
  \includegraphics[width=1.0\textwidth]{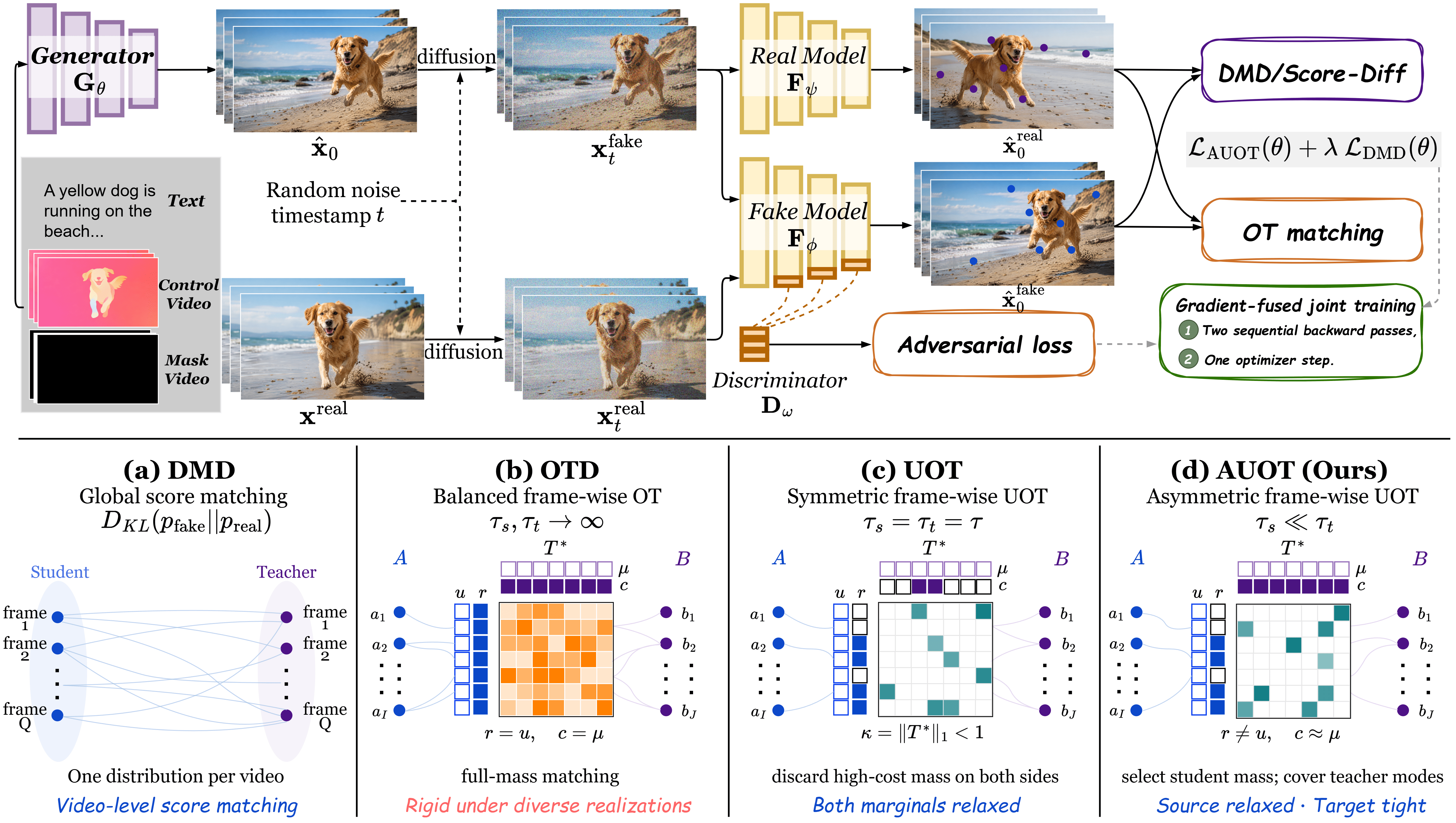}
  \caption{\textbf{Overview of \methodpp.} \emph{Top:} Given the conditioning
  signal, the few-step generator is jointly optimized by the DMD score-difference
  loss, frame-wise optimal-transport matching, and an adversarial
  loss through two sequential backward passes followed by one optimizer step.
  \emph{Bottom:} The OT diagrams show one paired temporal frame. For clarity,
  the diffusion-time and frame indices $(t,q)$ are omitted. DMD provides a
  score-difference signal at corresponding spatiotemporal coordinates;
  balanced OTD explicitly couples spatial tokens within each paired frame but
  enforces rigid full-mass matching; and symmetric UOT relaxes both marginals.
  Our asymmetric UOT instead uses a loose student-side penalty and a tight
  teacher-side penalty, selectively suppressing unreliable student-token mass
  while retaining coverage of teacher tokens.}
  \label{fig:framework}
  \vspace{0.2in}
\end{figure*}

\subsection{Notation and Problem Setup}
\label{sec:notation}

\begin{table}[t]
  \centering
  \caption{Core notation. The complete list is in Appendix~\ref{sec:supp_notation}.}
  \label{tab:notation_main}
  \footnotesize
  \setlength{\tabcolsep}{3pt}
  \begin{tabular}{@{}l p{0.63\columnwidth}@{}}
    \toprule
    Symbol & Meaning \\
    \midrule
    $G_\theta$; $F_\psi,F_\phi$ & generator; teacher and fake score networks \\
    $\pfake,\preal$ & student and teacher video distributions \\
    $\bm A_{t,q},\bm B_{t,q}$ & spatial-token sets in temporal frame $q$ at diffusion time $t$ \\
    $\bm u,\bm\mu$ & prescribed student and teacher marginals \\
    $\Cmat$ & pairwise ground-cost matrix \\
    $\Tplan_{t,q},\Topt_{t,q}$ & within-frame transport plan and its optimum \\
    $\bm r_{t,q},\bm c_{t,q}$ & row and column projections of $\Topt_{t,q}$ \\
    $\epsilon$; $\tau_s,\tau_t$ & entropy; marginal-penalty strengths \\
    $\rho_s,\rho_t$; $\kappa$ & damping factors; total plan mass \\
    \bottomrule
  \end{tabular}
\end{table}

\noindent\textbf{Networks.}
Training uses the few-step generator $G_\theta$, the frozen teacher score
network $F_\psi$, and the trainable fake score network $F_\phi$, which estimate
$s_{\real}$ and $s_{\fake}$, respectively. A lightweight discriminator head
$D_\omega$ is embedded in $F_\phi$ (Sec.~\ref{sec:gan}).

\noindent\textbf{Generation and noising.}
Given Gaussian noise $\bm{z}$ and condition $\bm{c}$, the generator produces
$\hat{\bm{x}}_0=G_\theta(\bm{z},\bm{c})$ in four steps. At a sampled diffusion
time $t$, we re-noise the output as
$\bm{x}_t=\alpha_t\hat{\bm{x}}_0+\sigma_t\bm{\xi}$ with
$\bm{\xi}\!\sim\!\mathcal{N}(\bm{0},\bm{I})$ under the base-model noise
schedule, and both score networks evaluate $\bm{x}_t$.

\noindent\textbf{Frame-wise spatial-token empirical distributions.}
The score predictions retain their temporal and spatial axes. For each frame
$q$ at diffusion time $t$, we spatially pool the feature maps and flatten their
spatial axes into student tokens
$\bm{A}_{t,q}=[\bm{a}_{t,q,1};\ldots;\bm{a}_{t,q,I}]
\in\mathbb{R}^{I\times D}$ and teacher tokens
$\bm{B}_{t,q}=[\bm{b}_{t,q,1};\ldots;\bm{b}_{t,q,J}]
\in\mathbb{R}^{J\times D}$. With marginals $\bm{u}\in\Delta^{I-1}$ and
$\bm{\mu}\in\Delta^{J-1}$ (uniform by default), a ground metric $d$ defines
$C_{t,q,ij}=d(\bm{a}_{t,q,i},\bm{b}_{t,q,j})$ and hence the within-frame cost
matrix $\Cmat_{t,q}$. Its transport plan is $\Tplan_{t,q}$, with optimum
$\Topt_{t,q}$. We solve transport independently for each frame, and omit $t,q$ when unambiguous.

\subsection{Optimal Transport Distillation}
\label{sec:otd}

We use the gradient of the EOT objective in Eq.~\eqref{eq:eot} as a geometric
distillation signal between $\pfake$ and $\preal$. A classical envelope
result makes this signal practical: the derivative of the optimal transport
value with respect to each pairwise cost is simply the corresponding entry of
the optimal plan~\cite{peyre2019computational}.

\begin{proposition}[Envelope property of entropic OT]
\label{prop:envelope}
Fix $\epsilon>0$ and let $W(\Cmat)$ denote the optimal value of the entropic
problem~\eqref{eq:eot} (or of its unbalanced
generalization~\eqref{eq:auot} with $\tau_s,\tau_t>0$) as a function of the
cost matrix. Then $W$ is concave and differentiable, and
$\partial W/\partial C_{ij}=T^*_{ij}$, where $\Topt$ is the unique minimizer
of the corresponding problem.
\end{proposition}


\noindent Applying the chain rule through $C_{ij}=d(\bm{a}_i,\bm{b}_j)$, the
gradient of $W$ with respect to a student spatial token $\bm{a}_i$ is
\begin{equation}
  \nabla_{\bm{a}_i}W
  = \sum_{j=1}^{J} T^*_{ij}\,\nabla_{\bm{a}_i} d(\bm{a}_i,\bm{b}_j)
  \label{eq:otgrad}
\end{equation}
With the squared-$\ell_2$ ground cost $d(\bm{a},\bm{b})=\tfrac12\|\bm{a}-\bm{b}\|_2^2$, this becomes:
\begin{equation}
  \nabla_{\bm{a}_i}W
  = \sum_j T^*_{ij}\,(\bm{a}_i-\bm{b}_j).
  \label{eq:otgrad_l2}
\end{equation}
Each teacher token $\bm{b}_j$ pulls the student token $\bm{a}_i$ in a
direction and with a strength determined by $T^*_{ij}$. By the envelope
property, the implicit dependence of $\Topt$ on $\bm{a}_i$ contributes no
additional first-order term, so the Sinkhorn iterations need not be
differentiated through.

\noindent\textbf{Backpropagation to the generator.}
Using the OT gradient \eqref{eq:otgrad_l2} as a \emph{pseudo-gradient}
signal analogous to the KL score difference in DMD, we define
$\nabla_{\mathrm{OT}}(\bm{x}_t,t)$ as the result of backpropagating
Eq.~\eqref{eq:otgrad_l2} from $\bm{a}_i$ through the fake-score network to
$\bm{x}_t$. The OTD loss then mirrors the DMD loss~\eqref{eq:dmd}:
\begin{equation}
  \mathcal{L}_{\mathrm{OTD}}(\theta)
  = \mathbb{E}_{\bm{z},\bm{c},t,\bm{\xi}}\!\left[
    \big\|\hat{\bm{x}}_0 - \sg\big(\hat{\bm{x}}_0
      - \alpha_{\mathrm{OT}}\,\nabla_{\mathrm{OT}}(\bm{x}_t,t)\big)\big\|_2^2
  \right],
  \label{eq:otd}
\end{equation}
where $\alpha_{\mathrm{OT}}$ aligns the magnitudes of the OT and DMD signals.
Unlike the unconstrained score difference, the OT gradient couples student
and teacher spatial tokens through an explicit within-frame plan. The plans
are computed independently for all paired temporal frames and their gradients
are averaged over the batch and frame axes. OTD therefore supplies geometric
directions that reverse KL lacks when the two distributions have limited
overlap (Fig.~\ref{fig:reverse_kl_ot}). 
This objective, used with uniform marginals and the squared cost, is the
formulation of our conference version~\cite{wang2026vdot}.

\subsection{Why Balanced OTD Weakens under Output Diversity}
\label{sec:diagnosis}

Balanced OTD is reliable when dense conditions align spatial content within
corresponding frames, but restrictive for T2V and single-frame I2V, where two
valid videos need not share the same spatial layout. Its balanced,
squared-$\ell_2$ formulation can therefore match unrelated tokens and
aggregate diverse targets poorly.

\noindent\textbf{Pathology A (whom to match): forced correspondence.}
The constraints $\Tplan\bm{1}=\bm{u}$ and
$\Tplan^{\!\top}\bm{1}=\bm{\mu}$ require every spatial token to send and
receive its prescribed mass. When two valid videos differ in spatial content,
unrelated token pairs still contribute displacement vectors to
Eq.~\eqref{eq:otgrad_l2}. Soft partial matching instead reduces the influence
of costly token pairs and marginals.

\noindent\textbf{Pathology B (how to aggregate): mean-based averaging.}
The ground cost determines how each student token aggregates its matched
teacher tokens.

\begin{proposition}[Squared cost pulls toward the weighted mean]
\label{prop:mean}
For a fixed transport plan $T^*$, consider any row with positive
transported mass $r_i=\sum_j T^*_{ij}>0$, 
the squared-$\ell_2$ OT gradient~\eqref{eq:otgrad_l2} is proportional to the difference between the student token and the transport-weighted mean of the teacher tokens:
\begin{align}
  \nabla_{\bm{a}_i}W
  &= \sum_j T^*_{ij}\,(\bm{a}_i-\bm{b}_j)
  \notag\\
  &= r_i\Big(\bm{a}_i
     - \underbrace{\frac{\sum_j T^*_{ij}\,\bm{b}_j}{\sum_j T^*_{ij}}}_{
       \bar{\bm{b}}_i\;(\text{weighted mean})}\Big),
  \label{eq:meanpull}
\end{align}
where $r_i$ is the total mass leaving $\bm{a}_i$.
Consequently, $\nabla_{\bm{a}_i}W=\bm{0}$ if and only if
$\bm{a}_i=\bar{\bm{b}}_i$, the weighted mean of its transport targets under
the plan at that point.
\end{proposition}


For a fixed transport plan, visually distinct targets can yield a weighted
mean between modes and blur localized high-frequency features. Marginal
relaxation and a robust cost are therefore complementary: the former selects
reliable pairs, while the latter changes how their targets are aggregated.

\subsection{Partial Matching: Addressing \emph{Whom to Match}}
\label{sec:partial}

We use unbalanced OT to implement soft mass reweighting. Specifically, the
hard equalities
$\Tplan\bm{1}=\bm{u}$ and $\Tplan^{\!\top}\bm{1}=\bm{\mu}$ are replaced by
soft KL penalties, where $\mathrm{KL}$ denotes the generalized KL divergence
for nonnegative measures (Appendix~\ref{sec:supp_impl}). The symmetric
relaxed problem is
\begin{multline}
  \mathcal{D}_{\mathrm{UOT}}^{\epsilon,\tau}(\bm{A},\bm{B})
  = \min_{\Tplan\ge 0}\;
    \langle \Cmat,\Tplan\rangle + \epsilon\mathcal{E}(\Tplan) \\
    + \tau\big[\mathrm{KL}(\Tplan\bm{1}\,\|\,\bm{u})
             + \mathrm{KL}(\Tplan^{\!\top}\bm{1}\,\|\,\bm{\mu})\big].
  \label{eq:uot}
\end{multline}

\noindent\textbf{Role of the penalty weight $\tau$.}
The parameter $\tau>0$ interpolates between two extremes. As
$\tau\!\to\!0$, the marginal penalties vanish and the solution approaches
the uncoupled Gibbs plan, so it no longer meaningfully preserves the source
or target mass. At moderate $\tau$, the plan can softly suppress costly rows,
columns, and pairwise couplings rather than enforcing the prescribed
marginals exactly. We monitor its total optimized mass by
\begin{equation}
  \kappa(\tau) = \sum_{i,j} T^*_{ij} \;>\;0,
  \label{eq:kappa}
\end{equation}
where the prescribed marginals have unit total mass. 

\begin{proposition}[Recovery of balanced OT]
\label{prop:limit}
Let $\Tplan_\tau$ denote the minimizer of~\eqref{eq:uot} (or of the
asymmetric variant~\eqref{eq:auot} with
$\tau_s,\tau_t\!\to\!\infty$ jointly). Then, as $\tau\!\to\!\infty$,
\begin{equation}
  \mathcal{D}_{\mathrm{UOT}}^{\epsilon,\tau}(\bm{A},\bm{B})\;\nearrow\;
  \Weps(\bm{A},\bm{B}),
  \qquad
  \Tplan_\tau \to \Topt_{\mathrm{bal}},
\end{equation}
where $\Topt_{\mathrm{bal}}$ is the unique optimizer of the balanced
entropic problem~\eqref{eq:eot}.
\end{proposition}

The proof is provided in Appendix~\ref{sec:supp_proofs}.

\noindent\textbf{Solving the relaxed problem.}
Eq.~\eqref{eq:uot} is solved with the scaling algorithm of Chizat et
al.~\cite{chizat2018unbalanced}, a Sinkhorn-type block-coordinate ascent on
the concave dual whose only change relative to balanced Sinkhorn is that
each log-domain dual update is damped by the factor
$\rho=\tau/(\tau+\epsilon)\in(0,1)$; the limit $\rho\to1$ recovers the
balanced iteration. Under the standard assumptions of generalized Sinkhorn scaling,
the iterates converge to the relaxed solution~\cite{chizat2018unbalanced};
a few dozen iterations suffice for our problem sizes
(Sec.~\ref{sec:setup}). The updates for the asymmetric variant we actually
use are given below.

Partial matching relaxes only the frame-wise spatial-token OTD term; DMD and
the adversarial loss still constrain the complete generated distribution.
Downweighted token mass therefore attenuates unreliable local correspondence
rather than global supervision.

\noindent\textbf{Asymmetric relaxation for distillation.}
The symmetric formulation in Eq.~\eqref{eq:uot} treats both marginals alike,
although the student is optimized and the teacher is fixed. Unreliable
student tokens should be allowed to carry reduced row mass, whereas difficult
teacher modes should remain represented. We therefore assign separate marginal
penalties:
\begin{multline}
  \mathcal{D}_{\mathrm{AUOT}}^{\epsilon,\tau_s,\tau_t}(\bm{A},\bm{B})
  = \min_{\Tplan\ge 0}\;
    \langle \Cmat,\Tplan\rangle + \epsilon\mathcal{E}(\Tplan) \\
    + \tau_s\,\mathrm{KL}(\Tplan\bm{1}\,\|\,\bm{u})
    + \tau_t\,\mathrm{KL}(\Tplan^{\!\top}\bm{1}\,\|\,\bm{\mu}),
  \label{eq:auot}
\end{multline}
where larger $\tau_s$ and $\tau_t$ impose tighter student and teacher
marginals, respectively. We use a small $\tau_s$ to attenuate unreliable
student gradients and a large $\tau_t$ to prevent the plan from strongly
downweighting difficult teacher content.

\noindent\textbf{Asymmetric Sinkhorn iterations.}
In the scaling algorithm~\cite{chizat2018unbalanced}, the asymmetry simply
assigns separate damping factors to the row and column updates. Defining
\begin{equation}
  \rho_s = \frac{\tau_s}{\tau_s + \epsilon},\qquad
  \rho_t = \frac{\tau_t}{\tau_t + \epsilon},
  \label{eq:rho}
\end{equation}
the log-domain scaling updates read
\begin{align}
  f_i^{(\ell+1)} &= \rho_s\Big(\log u_i
    - \log\!\sum\nolimits_j
      \exp\big(\log K_{ij}+g_j^{(\ell)}\big)\Big),
  \label{eq:asink_f}\\[2pt]
  g_j^{(\ell+1)} &= \rho_t\Big(\log \mu_j
    - \log\!\sum\nolimits_i
      \exp\big(\log K_{ij}+f_i^{(\ell+1)}\big)\Big),
  \label{eq:asink_g}
\end{align}
where $K_{ij}=\exp(-C_{ij}/\epsilon)$ is the Gibbs kernel,
$\bm{f}\in\mathbb{R}^I$ and $\bm{g}\in\mathbb{R}^J$ are log-domain
scaling variables (equivalently, $\epsilon$-rescaled dual potentials), and
the plan is recovered as $T^*_{ij}=\exp(f_i+g_j+\log K_{ij})$.
The updates follow a Gauss--Seidel order: the source scaling is updated from
$\bm{g}^{(\ell)}$, after which the target scaling is updated using the newly
computed $\bm{f}^{(\ell+1)}$. When $\rho_s=\rho_t=1$,
Eqs.~\eqref{eq:asink_f}--\eqref{eq:asink_g} reduce to the standard balanced
Sinkhorn iterations.
When $\rho_s$ is small (small $\tau_s$), the row update applies only a
fraction of the correction, allowing the source marginal to deviate from
$\bm{u}$; when $\rho_t$ is close to~1 (large $\tau_t$), the column update
nearly enforces $\bm{\mu}$, keeping the target marginal almost exact.
Setting $\tau_s\!=\!\tau_t$ recovers the symmetric case, and sending both
$\tau_s,\tau_t\!\to\!\infty$ recovers balanced OT
(Proposition~\ref{prop:limit}). This update is the asymmetric counterpart of
the generalized Sinkhorn scaling of~\cite{chizat2018unbalanced}.

\noindent\textbf{Gradient modulation.}
Equation~\eqref{eq:meanpull} shows that the gradient scale is proportional to
$r_i=\sum_jT^*_{ij}$. A student spatial token without a reliable match can
have $r_i<u_i$, attenuating its OT gradient, while the tight teacher marginal
maintains coverage of teacher tokens. The asymmetry thus encodes the roles of
the evolving student and fixed teacher directly in the plan.

\subsection{$\ell_1$ Ground Cost: Addressing \emph{How to Aggregate}}
\label{sec:l1}

We replace the squared-$\ell_2$ ground cost with the $\ell_1$ (Manhattan)
cost $d(\bm{a},\bm{b})=\|\bm{a}-\bm{b}\|_1$. By the envelope property, the
first-order derivative is evaluated at the optimal plan while holding that
plan fixed. Because the $\ell_1$ norm is
non-smooth, its sub-gradient is the component-wise sign function:
$\partial d/\partial\bm{a}=\mathrm{sign}(\bm{a}-\bm{b})$. Substituting into
the OT gradient~\eqref{eq:otgrad} gives
\begin{equation}
  \nabla_{\bm{a}_i}W
  = \sum_{j=1}^{J} T^*_{ij}\,\mathrm{sign}(\bm{a}_i-\bm{b}_j).
  \label{eq:medianpull}
\end{equation}
For a fixed transport plan, changing the ground cost changes the stationary
target from the weighted mean to the weighted median:

\begin{proposition}[$\ell_1$ cost pulls toward the weighted median]
\label{prop:median}
Fix a plan $\Tplan$, a row with positive transported mass
$r_i=\sum_jT_{ij}>0$, and a coordinate $d\in\{1,\ldots,D\}$. The convex
function $\varphi(a)=\sum_{j} T_{ij}\,|a - b_j^{(d)}|$ is minimized exactly
at the \emph{weighted medians} of $\{b_j^{(d)}\}$ with weights $\{T_{ij}\}$;
in particular, away from ties the gradient~\eqref{eq:medianpull} vanishes if
and only if the transported mass on the two sides of $a$ is equal.
\end{proposition}


\noindent\textbf{Why the median is robust under diversity.}
For a fixed plan, the weighted mean is sensitive to the magnitude of distant
targets, so a small amount of transported mass can exert a disproportionately
large pull. The coordinate-wise weighted median instead depends on how mass
is ordered around each coordinate and has bounded sensitivity to target
magnitude. When most transported mass forms a coherent cluster, the median
remains governed by that cluster rather than being displaced toward an
extreme target. The $\ell_1$ cost thus reduces averaging across incompatible
content, complementing partial matching: the latter reduces the influence of
unreliable pairs, while the former limits the influence of diverse targets
that remain in the plan.

\noindent\textbf{Behavior across conditioning strengths and final objective.}
Strong conditions align the spatial-token sets of corresponding student and
teacher frames, so the optimized marginals are expected to remain close to
their prescribed values and weighted means and medians become similar. The
relaxed objective therefore
approaches the behavior of balanced OTD when correspondences are reliable,
while remaining robust as the valid output space widens
(Sec.~\ref{sec:ablation}). The final matching objective is
\begin{equation}
  \mathcal{L}_{\mathrm{match}}(\theta)
  = w_{\mathrm{DMD}}\,\mathcal{L}_{\mathrm{DMD}}(\theta)
    + w_{\mathrm{OT}}\,\mathcal{L}_{\mathrm{OTD}}(\theta),
  \label{eq:match}
\end{equation}
where $\mathcal{L}_{\mathrm{OTD}}$ uses the asymmetric partial plan of
Eq.~\eqref{eq:auot} with the $\ell_1$ cost; the weights are shared across
all experiments (Sec.~\ref{sec:setup}).

\subsection{Adversarial Discriminator}
\label{sec:gan}

Relying on distribution matching alone, however, leaves score-approximation errors that surface as texture artifacts. We therefore introduce a relativistic
adversarial objective.

\noindent\textbf{Discriminator architecture.}
We adopt the register-token discriminator design of APT~\cite{lin2025apt},
also used by Self-Forcing~\cite{huang2025selfforcing}. A lightweight head
$D_\omega$ embeds register tokens in $F_\phi$, cross-attends to selected
backbone features, and predicts one logit without a separate encoder.
Architecture details are given in Appendix~\ref{sec:supp_impl}.

\noindent\textbf{Relativistic GAN losses.}
Following the relativistic framework~\cite{jolicoeur2018relativistic}, the
generator and discriminator losses are defined relative to the difference
between real and fake logits:
\begin{align}
  \mathcal{L}_{\mathrm{GAN\text{-}G}}(\theta)
  &= \mathbb{E}_{\bm{z},t}\!\Big[
    \softplus\!\big(D_\omega(\bm{x}_t^{\real},t)-D_\omega(\bm{x}_t^{\fake},t)\big)
  \Big],
  \label{eq:gan_g}\\[3pt]
  \mathcal{L}_{\mathrm{GAN\text{-}D}}(\omega)
  &= \mathbb{E}_{\bm{z},t}\!\Big[
    \softplus\!\big(D_\omega(\bm{x}_t^{\fake},t)-D_\omega(\bm{x}_t^{\real},t)\big)
  \Big],
  \label{eq:gan_d}
\end{align}
where $\softplus(x)=\log(1+e^x)$, and $\bm{x}_t^{\fake}$ and
$\bm{x}_t^{\real}$ are generated and real videos noised to the same~$t$.
This formulation compares their logits within each batch.
We additionally use the finite-difference R1 penalty of APT
~\cite{mescheder2018training,lin2025apt}; its definition and parameters are
provided in Appendix~\ref{sec:supp_impl}.

\subsection{Model Learning}
\label{sec:learning}

The full training procedure operates over three networks: the generator
$G_\theta$, the fake-score network $F_\phi$ (with embedded discriminator
head $D_\omega$), and the frozen teacher score $F_\psi$. The critic is
updated at every step and the generator every $R$ critic steps, each with a
joint objective that combines its distillation and adversarial terms. We use
task-specific indicators $a_G(\mathrm{step}),a_D(\mathrm{step})\in\{0,1\}$
to encode the adversarial warm-up and update intervals for the generator and
critic, respectively. The overall generator objective is:
\begin{equation}
  \begin{aligned}
  \mathcal{L}_{G}(\theta;\mathrm{step})
  ={}& w_{\mathrm{DMD}}\,\mathcal{L}_{\mathrm{DMD}}(\theta)
    + w_{\mathrm{OT}}\,\mathcal{L}_{\mathrm{OTD}}(\theta) \\
    &+ a_G(\mathrm{step})\,w_{\mathrm{GAN\text{-}G}}\,
      \mathcal{L}_{\mathrm{GAN\text{-}G}}(\theta),
  \end{aligned}
  \label{eq:total}
\end{equation}
where OTD uses Eq.~\eqref{eq:auot} with the $\ell_1$ cost and the GAN weight
is kept small so it acts as refinement. The critic adds
Eq.~\eqref{eq:gan_d} when $a_D(\mathrm{step})=1$ to the denoising loss:
\begin{equation}
  \mathcal{L}_{\text{denoise}}(\phi)
  = \mathbb{E}_{\bm{z},\bm{c},t,\bm{\xi}}\!\Big[
    \big\|\hat{\bm{x}}_0 - F_\phi(\bm{x}_t,t)\big\|_2^2
  \Big],
  \label{eq:denoise}
\end{equation}
The transport plan is computed only during training. The task-specific
definitions of $a_G$ and $a_D$, including their warm-up and update intervals,
are provided in Appendix~\ref{sec:supp_impl}. Algorithm~\ref{alg:training}
summarizes the resulting framework.

\noindent\textbf{Fused updates under memory constraints.}
\label{sec:gradfuse}
At the 14B scale, we avoid retaining two live computation graphs by
backpropagating the distillation loss, freeing its graph, and then re-running
the same sample for the adversarial loss on scheduled adversarial updates.
The two gradients accumulate before one optimizer step, preserving the joint
objective with the peak memory of a single backward pass. The critic is
handled analogously. This
\emph{gradient-fused} schedule replaces the alternating updates of the
conference version~\cite{wang2026vdot} and is ablated in Sec.~\ref{sec:ablation}.

\begin{algorithm}[t]
\SetAlgoLined
\KwIn{Generator $G_\theta$, teacher $F_\psi$, fake-score $F_\phi$
with discriminator head $D_\omega$, critic-update ratio $R$, task-specific
adversarial schedules $a_G,a_D$}
\For{$\mathrm{step}=1,2,\ldots$}{
  Sample data $(\bm{c},\bm{x}^{\real})$ and noise $\bm{z}$\;
  \tcp{Generator update}
  \If{$\mathrm{step} \bmod R = 0$}{
    $\hat{\bm{x}}_0 \leftarrow G_\theta(\bm{z},\bm{c})$\;
    Compute $\Topt_{t,q}$ for every frame $q$ via Sinkhorn
      (Eqs.~\ref{eq:asink_f}--\ref{eq:asink_g})\;
    $\mathcal{L}_{\text{match}}$ (Eq.~\ref{eq:match})$.\text{backward}()$;
      free graph\;
    \If{$a_G(\mathrm{step})=1$}{
      Re-run $\hat{\bm{x}}_0 \leftarrow G_\theta(\bm{z},\bm{c})$
        \tcp*{same $\bm{z}$}
      $w_{\text{GAN-G}}\,\mathcal{L}_{\text{GAN-G}}$
        (Eq.~\ref{eq:gan_g})$.\text{backward}()$ \tcp*{accum.}
    }
    Update $G_\theta$\;
  }
  \tcp{Critic update}
  $\hat{\bm{x}}_0 \leftarrow G_\theta(\bm{z},\bm{c})$ \tcp*{no grad}
  $\mathcal{L}_{\text{denoise}}$ (Eq.~\ref{eq:denoise})$.\text{backward}()$;
    free graph\;
  \If{$a_D(\mathrm{step})=1$}{
    $\big(w_{\text{GAN-D}}\,\mathcal{L}_{\text{GAN-D}}
      \text{ (Eq.~\ref{eq:gan_d})}
      + \mathcal{L}_{\mathrm{R1}}\text{ (Appendix~\ref{sec:supp_impl})}\big)
      .\text{backward}()$ \tcp*{accum.}
  }
  Update $F_\phi, D_\omega$\;
}
\caption{Training procedure of \methodpp.}
\label{alg:training}
\end{algorithm}

\begin{table*}[t]
\caption{
\textbf{Quantitative comparison on representative single-condition tasks of
\methodbench}. We bold the \textbf{best results} and underline the
\underline{second-best results}.
}

\centering
\footnotesize
\setlength{\tabcolsep}{5.8pt}
\begin{tabular}{l|l|c|c|ccccccc}

\toprule

Type & Method & Base Model & \#NFE$\downarrow$ & \multicolumn{7}{c}{\textbf{Video Quality \& Video Consistency}} 
 \\

& & & & \rotatebox{90}{\begin{minipage}{1.8cm} {\tikz\fill[video_quality] (0,0) circle (.5ex);} \centering Aesthetic Quality\end{minipage}}
& \rotatebox{90}{\begin{minipage}{1.8cm} {\tikz\fill[video_quality] (0,0) circle (.5ex);} \centering Background Consistency\end{minipage}}
& \rotatebox{90}{\begin{minipage}{1.8cm} {\tikz\fill[video_quality] (0,0) circle (.5ex);} \centering Dynamic Degree\end{minipage}}
& \rotatebox{90}{\begin{minipage}{1.8cm} {\tikz\fill[video_quality] (0,0) circle (.5ex);} \centering Imaging Quality\end{minipage}}
& \rotatebox{90}{\begin{minipage}{1.8cm} {\tikz\fill[video_quality] (0,0) circle (.5ex);} \centering Motion Smoothness\end{minipage}}
& \rotatebox{90}{\begin{minipage}{1.8cm} {\tikz\fill[video_quality] (0,0) circle (.5ex);} \centering Subject Consistency\end{minipage}}
& \rotatebox{90}{\begin{minipage}{1.8cm} {\tikz\fill[score_average] (0,0) circle (.5ex);} \centering Normalized Average\end{minipage}}
 \\

\midrule 

\multirow{6}{*}{Pose} 
& ControlVideo~\cite{zhang2023controlvideo} &SD-1.5  &100  &63.32\%  &\textbf{96.82\%}  &10.00\%  &69.88\%  &98.50\%  &94.45\%  &72.16\%   \\
& Follow-Your-Pose~\cite{ma2024follow} &SD-1.4  &50  &50.36\%  &88.61\%  &40.00\%  &68.73\%  &91.78\%  &77.81\%  &69.55\%  \\ 
& VACE~\cite{jiang2025vace} &LTX-0.9B  &80  &59.72\%  &95.67\%  &45.00\%  &69.12\%  &98.10\%  &\textbf{94.80\%}  &77.06\%  \\
& VACE~\cite{jiang2025vace} &Wan-14B  &100  &\textbf{64.99\%} &95.48\%  &\underline{55.00\%}  &69.15\%  &\underline{98.64\%}  &94.07\%  &79.56\%   \\
& \cellcolor{tabhighlight}\methodbf (Ours)  &\cellcolor{tabhighlight}Wan-14B  &\cellcolor{tabhighlight}4  &\cellcolor{tabhighlight}63.50\%  &\cellcolor{tabhighlight}95.88\%  &\cellcolor{tabhighlight}\textbf{60.00\%}  &\cellcolor{tabhighlight}\textbf{70.65\%}  &\cellcolor{tabhighlight}\textbf{98.69\%}  &\cellcolor{tabhighlight}\underline{94.53\%}  &\cellcolor{tabhighlight}\underline{80.54\%}   \\

& \cellcolor{tabhighlight}\methodppbf (Ours)  &\cellcolor{tabhighlight}Wan-14B  &\cellcolor{tabhighlight}4  &\cellcolor{tabhighlight}\underline{64.97\%}  &\cellcolor{tabhighlight}\underline{95.94\%}  &\cellcolor{tabhighlight}\textbf{60.00\%} &\cellcolor{tabhighlight}\underline{70.14\%}  &\cellcolor{tabhighlight}98.55\%  &\cellcolor{tabhighlight}93.73\%  &\cellcolor{tabhighlight}\textbf{80.55\%}   \\

\midrule 

\multirow{5}{*}{Flow} 
& FLATTEN~\cite{cong2023flatten} &SD-2.1  &100  &\textbf{63.86\%}  &\textbf{96.68\%}  &60.00\%  &53.28\%  &96.35\%  &\underline{93.39\%}  &77.26\%   \\ 
& VACE~\cite{jiang2025vace} &LTX-0.9B  &80  &55.47\%  &95.97\%  &65.00\%  &57.51\%  &97.62\%  &92.90\%  &77.41\%  \\
& VACE~\cite{jiang2025vace} &Wan-14B  &100  &\underline{62.78\%}  &95.78\%  &\textbf{75.00\%}  &58.01\%  &97.77\%  &92.75\%  &\textbf{80.35\%}   \\
& \cellcolor{tabhighlight}\methodbf (Ours)  &\cellcolor{tabhighlight}Wan-14B  &\cellcolor{tabhighlight}4  &\cellcolor{tabhighlight}\underline{62.78\%}  &\cellcolor{tabhighlight}\underline{96.36\%}  &\cellcolor{tabhighlight}\underline{70.00\%}  &\cellcolor{tabhighlight}\textbf{59.22\%}  &\cellcolor{tabhighlight}\textbf{98.84\%} &\cellcolor{tabhighlight}\textbf{93.93\%}  &\cellcolor{tabhighlight}80.18\%  \\

& \cellcolor{tabhighlight}\methodppbf (Ours)  &\cellcolor{tabhighlight}Wan-14B  &\cellcolor{tabhighlight}4  &\cellcolor{tabhighlight}61.52\%  &\cellcolor{tabhighlight}95.63\%  &\cellcolor{tabhighlight}\textbf{75.00\%} &\cellcolor{tabhighlight}\underline{58.23\%}  &\cellcolor{tabhighlight}\underline{98.23\%}  &\cellcolor{tabhighlight}93.11\%  &\cellcolor{tabhighlight}\underline{80.29\%}   \\

\midrule 

\multirow{4}{*}{Grey} 
& VACE~\cite{jiang2025vace} &LTX-0.9B  &80  &60.64\%  &98.09\%  &5.00\%  &58.29\%  &\textbf{99.37\%}  &97.14\%  &69.75\%  \\
& VACE~\cite{jiang2025vace} &Wan-14B  &100  &\underline{63.76\%}  &\underline{98.11\%}  &\textbf{15.00\%}  &\textbf{60.90\%}  &99.30\%  &97.56\%  &\textbf{72.44\%}   \\
& \cellcolor{tabhighlight}\methodbf (Ours)  &\cellcolor{tabhighlight}Wan-14B  &\cellcolor{tabhighlight}4  &\cellcolor{tabhighlight}62.13\%  &\cellcolor{tabhighlight}\underline{98.11\%}  &\cellcolor{tabhighlight}\textbf{15.00\%}  &\cellcolor{tabhighlight}\underline{60.60\%}  &\cellcolor{tabhighlight}\underline{99.34\%}  &\cellcolor{tabhighlight}\underline{97.62\%}  &\cellcolor{tabhighlight}72.13\%   \\

& \cellcolor{tabhighlight}\methodppbf (Ours)  &\cellcolor{tabhighlight}Wan-14B  &\cellcolor{tabhighlight}4  &\cellcolor{tabhighlight}\textbf{63.84\%}  &\cellcolor{tabhighlight}\textbf{98.36\%}  &\cellcolor{tabhighlight}\textbf{15.00\%} &\cellcolor{tabhighlight}60.25\%  &\cellcolor{tabhighlight}99.33\%  &\cellcolor{tabhighlight}\textbf{97.65\%}  &\cellcolor{tabhighlight}\underline{72.41\%}   \\

\midrule 

\multirow{5}{*}{Outpaint} 
& Follow-Your-Canvas~\cite{chen2024follow}  &SD-2.1  &80  &51.98\%  &97.13\%  &25.00\%  &66.95\%  &98.92\%  &95.55\%  &72.59\%   \\ 
& VACE~\cite{jiang2025vace} &LTX-0.9B  &80  &58.78\%  &98.07\%  &25.00\%  &70.74\%  &\textbf{99.20\%}  &\underline{97.35\%}  &74.86\%  \\
& VACE~\cite{jiang2025vace} &Wan-14B  &100  &\underline{62.66\%}  &\underline{98.13\%}  &25.00\%  &70.56\%  &\underline{99.18\%}  &\textbf{97.41\%}  &75.49\%  \\
& \cellcolor{tabhighlight}\methodbf (Ours)  &\cellcolor{tabhighlight}Wan-14B  &\cellcolor{tabhighlight}4  &\cellcolor{tabhighlight}\textbf{63.18\%}  &\cellcolor{tabhighlight}\textbf{98.22\%}  &\cellcolor{tabhighlight}\underline{30.00\%}  &\cellcolor{tabhighlight}\textbf{71.38\%}  &\cellcolor{tabhighlight}99.16\%  &\cellcolor{tabhighlight}97.34\%  &\cellcolor{tabhighlight}\underline{76.54\%}   \\

& \cellcolor{tabhighlight}\methodppbf (Ours)  &\cellcolor{tabhighlight}Wan-14B  &\cellcolor{tabhighlight}4  &\cellcolor{tabhighlight}60.95\%  &\cellcolor{tabhighlight}97.99\%  &\cellcolor{tabhighlight}\textbf{40.00\%} &\cellcolor{tabhighlight}\underline{71.33\%}  &\cellcolor{tabhighlight}99.15\%  &\cellcolor{tabhighlight}96.83\%  &\cellcolor{tabhighlight}\textbf{77.71\%}   \\

\midrule

\multirow{6}{*}{R2V} 
& Keling1.6 &-  &-  &63.61\%  &93.85\%  &\textbf{90.00\%}  &64.46\%  &\textbf{98.95\%}  &90.11\%  &\textbf{83.50\%}   \\
 &Vidu2.0 &-  &-  &60.79\%  &93.22\%  &\underline{85.00\%}  &54.26\%  &97.54\%  &87.42\%  &79.71\%   \\
 &VACE~\cite{jiang2025vace} &LTX-0.9B  &80  &50.18\%  &91.31\%  &50.00\%  &55.88\%  &\underline{98.84\%}  &85.49\%  &71.95\%   \\
 &VACE~\cite{jiang2025vace} &Wan-14B  &100  &67.39\%  &\underline{95.39\%}  &80.00\%  &62.82\%  &97.96\%  &\underline{91.70\%}  &82.54\%  \\
& \cellcolor{tabhighlight}\methodbf (Ours)  &\cellcolor{tabhighlight}Wan-14B  &\cellcolor{tabhighlight}4  &\cellcolor{tabhighlight}\textbf{69.70\%}  &\cellcolor{tabhighlight}94.34\%  &\cellcolor{tabhighlight}70.00\%  &\cellcolor{tabhighlight}\underline{65.93\%}  &\cellcolor{tabhighlight}98.15\%  &\cellcolor{tabhighlight}89.88\%  &\cellcolor{tabhighlight}81.32\%   \\

& \cellcolor{tabhighlight}\methodppbf (Ours)  &\cellcolor{tabhighlight}Wan-14B  &\cellcolor{tabhighlight}4  &\cellcolor{tabhighlight}\underline{69.33\%}  &\cellcolor{tabhighlight}\textbf{97.45\%}  &\cellcolor{tabhighlight}70.00\% &\cellcolor{tabhighlight}\textbf{67.04\%}  &\cellcolor{tabhighlight}98.73\%  &\cellcolor{tabhighlight}\textbf{95.08\%}  &\cellcolor{tabhighlight}\underline{82.94\%}   \\

%
\bottomrule
\end{tabular}
\label{tab:comp1}
\end{table*}

\begin{table*}[t]
\caption{
\textbf{Quantitative evaluations on VBench (T2V)}. We compare \methodpp\ with
the many-step teacher and few-step baselines. We bold the \textbf{best results} and underline the \underline{second-best results}.
}
\centering
\footnotesize
\setlength{\tabcolsep}{6.5pt}
\begin{tabular}{l|c|c|cccccc|c}

\toprule

Method & Base Model & \#NFE$\downarrow$ & \multicolumn{7}{c}{\textbf{Video Quality \& Video Consistency}} \\

& & & \rotatebox{90}{\begin{minipage}{1.8cm} {\tikz\fill[video_quality] (0,0) circle (.5ex);} \centering Aesthetic Quality\end{minipage}}
& \rotatebox{90}{\begin{minipage}{1.8cm} {\tikz\fill[video_quality] (0,0) circle (.5ex);} \centering Background Consistency\end{minipage}}
& \rotatebox{90}{\begin{minipage}{1.8cm} {\tikz\fill[video_quality] (0,0) circle (.5ex);} \centering Dynamic Degree\end{minipage}}
& \rotatebox{90}{\begin{minipage}{1.8cm} {\tikz\fill[video_quality] (0,0) circle (.5ex);} \centering Imaging Quality\end{minipage}}
& \rotatebox{90}{\begin{minipage}{1.8cm} {\tikz\fill[video_quality] (0,0) circle (.5ex);} \centering Motion Smoothness\end{minipage}}
& \rotatebox{90}{\begin{minipage}{1.8cm} {\tikz\fill[video_quality] (0,0) circle (.5ex);} \centering Subject Consistency\end{minipage}}
& \rotatebox{90}{\begin{minipage}{1.8cm} {\tikz\fill[score_average] (0,0) circle (.5ex);} \centering Normalized Average\end{minipage}} \\

\midrule

VideoCrafter2~\cite{chen2024videocrafter2} &1.4B &50$\times$2 &66.34 &\textbf{98.48} &17.48 &\underline{68.46} &98.42 &\underline{98.16} &74.56 \\
LTX-Video~\cite{hacohen2024ltx} &13B &40 &\underline{66.80} &\underline{97.76} &23.93 &\textbf{69.97} &\underline{99.07} &\textbf{98.42} &75.99 \\
CogVideoX1.5~\cite{cogvideox} &5B &50$\times$2 &62.81 &94.95 &46.11 &62.81 &96.96 &92.66 &76.05 \\
HunyuanVideo~\cite{kong2024hunyuanvideo} &13B &50$\times$2 &65.11 &95.16 &\underline{47.67} &65.70 &\underline{99.07} &92.66 &77.56 \\
HunyuanVideo-1.5~\cite{hunyuanvideo2025} &8.3B &50$\times$2 &65.50 &94.72 &\textbf{80.98} &67.55 &\textbf{99.10} &92.31 &\textbf{83.36} \\
Wan2.1-T2V~\cite{wan2025} &Wan-1.3B &50$\times$2 &66.26 &97.16 &33.89 &67.49 &98.52 &96.91 &76.70 \\
Wan2.1-T2V~\cite{wan2025} (teacher) &Wan-14B &50$\times$2 &\textbf{68.76} &97.49 &36.20 &68.20 &98.59 &96.99 &\underline{77.70} \\ \hline

CausVid~\cite{yin2025causvid} &Wan-14B &4 &68.54 &97.17 &34.17 &69.14 &98.36 &\textbf{97.34} &77.45 \\
AccVideo~\cite{zhang2025accvideo} &Wan-14B &4 &\underline{69.32} &\textbf{97.85} &30.06 &62.29 &\underline{98.85} &97.13 &75.92 \\
rCM~\cite{zheng2025large} &Wan-14B &4 &\textbf{69.87} &97.20 &32.54 &69.30 &98.66 &96.92 &77.42 \\
DMD2~\cite{yin2024improved} &Wan-14B &4 &68.00 &96.16 &\underline{40.66} &\underline{70.54} &98.36 &95.38 &\underline{78.18} \\
\methodbf~\cite{wang2026vdot} &Wan-14B &4 &69.03 &\underline{97.73} &32.52 &70.25 &\textbf{99.01} &\underline{97.18} &77.62 \\
\cellcolor{tabhighlight}\methodppbf\ (Ours) &\cellcolor{tabhighlight}Wan-14B &\cellcolor{tabhighlight}4 &\cellcolor{tabhighlight}69.14&\cellcolor{tabhighlight}97.22 &\cellcolor{tabhighlight}\textbf{54.91} &\cellcolor{tabhighlight}\textbf{72.83} &\cellcolor{tabhighlight}98.56 &\cellcolor{tabhighlight}96.89 &\cellcolor{tabhighlight}\textbf{81.59} \\

\bottomrule
\end{tabular}
\label{tab:t2v_main}
\end{table*}

\begin{table*}[t]
\caption{
\textbf{Quantitative evaluations on VBench-I2V}. We compare \methodpp\ with
the many-step teacher and few-step baselines. We bold the \textbf{best results} and underline the \underline{second-best results}.
}
\centering
\footnotesize
\setlength{\tabcolsep}{6.5pt}
\begin{tabular}{l|c|c|cccccc|c}

\toprule

Method & Base Model & \#NFE$\downarrow$ & \multicolumn{7}{c}{\textbf{Video Quality \& Video Consistency}} \\

& & & \rotatebox{90}{\begin{minipage}{1.8cm} {\tikz\fill[video_quality] (0,0) circle (.5ex);} \centering Aesthetic Quality\end{minipage}}
& \rotatebox{90}{\begin{minipage}{1.8cm} {\tikz\fill[video_quality] (0,0) circle (.5ex);} \centering Background Consistency\end{minipage}}
& \rotatebox{90}{\begin{minipage}{1.8cm} {\tikz\fill[video_quality] (0,0) circle (.5ex);} \centering Dynamic Degree\end{minipage}}
& \rotatebox{90}{\begin{minipage}{1.8cm} {\tikz\fill[video_quality] (0,0) circle (.5ex);} \centering Imaging Quality\end{minipage}}
& \rotatebox{90}{\begin{minipage}{1.8cm} {\tikz\fill[video_quality] (0,0) circle (.5ex);} \centering Motion Smoothness\end{minipage}}
& \rotatebox{90}{\begin{minipage}{1.8cm} {\tikz\fill[video_quality] (0,0) circle (.5ex);} \centering Subject Consistency\end{minipage}}
& \rotatebox{90}{\begin{minipage}{1.8cm} {\tikz\fill[score_average] (0,0) circle (.5ex);} \centering Normalized Average\end{minipage}} \\

\midrule

SEINE~\cite{chen2023seine} &0.9B &250$\times$2 &59.21 &95.75 &25.63 &\underline{70.55} &96.65 &95.02 &73.80 \\
I2VGen-XL~\cite{zhang2023i2vgen} &1.4B &50$\times$2 &65.45 &\textbf{97.27} &20.28 &69.49 &\underline{98.43} &\textbf{96.65} &74.60 \\
DynamiCrafter~\cite{xing2024dynamicrafter} &1.4B &50$\times$2 &\textbf{66.38} &95.39 &31.97 &69.09 &97.46 &95.39 &75.94 \\
CogVideoX1.5~\cite{cogvideox} &5B &50$\times$2 &62.20 &95.13 &34.49 &69.44 &98.35 &94.53 &75.69 \\
Step-Video-TI2V~\cite{huang2025step}  &30B &50$\times$2 &63.74 &96.12 &37.34 &68.38 &97.40 &95.36 &76.39 \\
HunyuanVideo-1.5~\cite{hunyuanvideo2025} &8.3B &50$\times$2 &\underline{65.78} &\underline{96.24} &\textbf{50.85} &\textbf{70.86} &\textbf{98.52} &\underline{96.33} &\textbf{79.76} \\
Wan2.1-I2V~\cite{wan2025} (teacher) &Wan-14B &50$\times$2 &65.03 &96.05 &\underline{42.54} &70.47 &98.19 &95.23 &\underline{77.92} \\ \hline

CausVid~\cite{yin2025causvid} &Wan-14B &4 &64.87 &95.03 &38.12 &\underline{70.85} &97.35 &94.32 &\underline{76.76} \\
AccVideo~\cite{zhang2025accvideo} &Wan-14B &4 &63.09 &\underline{96.02} &18.21 &69.28 &\underline{97.88} &\underline{95.62} &73.35 \\
DMD2~\cite{yin2024improved} &Wan-14B &4 &63.74 &95.00 &\underline{38.59} &69.64 &97.47 &94.00 &76.40 \\
\methodbf~\cite{wang2026vdot} &Wan-14B &4 &\textbf{65.42} &\textbf{96.95} &21.41 &70.07 &\textbf{98.53} &\textbf{96.76} &74.86 \\
\cellcolor{tabhighlight}\methodppbf\ (Ours) &\cellcolor{tabhighlight}Wan-14B &\cellcolor{tabhighlight}4 &\cellcolor{tabhighlight}\underline{65.25} &\cellcolor{tabhighlight}95.38 &\cellcolor{tabhighlight}\textbf{45.07} &\cellcolor{tabhighlight}\textbf{71.64} &\cellcolor{tabhighlight}97.91 &\cellcolor{tabhighlight}95.03 &\cellcolor{tabhighlight}\textbf{78.38} \\

\bottomrule
\end{tabular}
\label{tab:i2v_main}
\end{table*}

\begin{figure*}[t]
  \centering
  \includegraphics[width=1.0\textwidth]{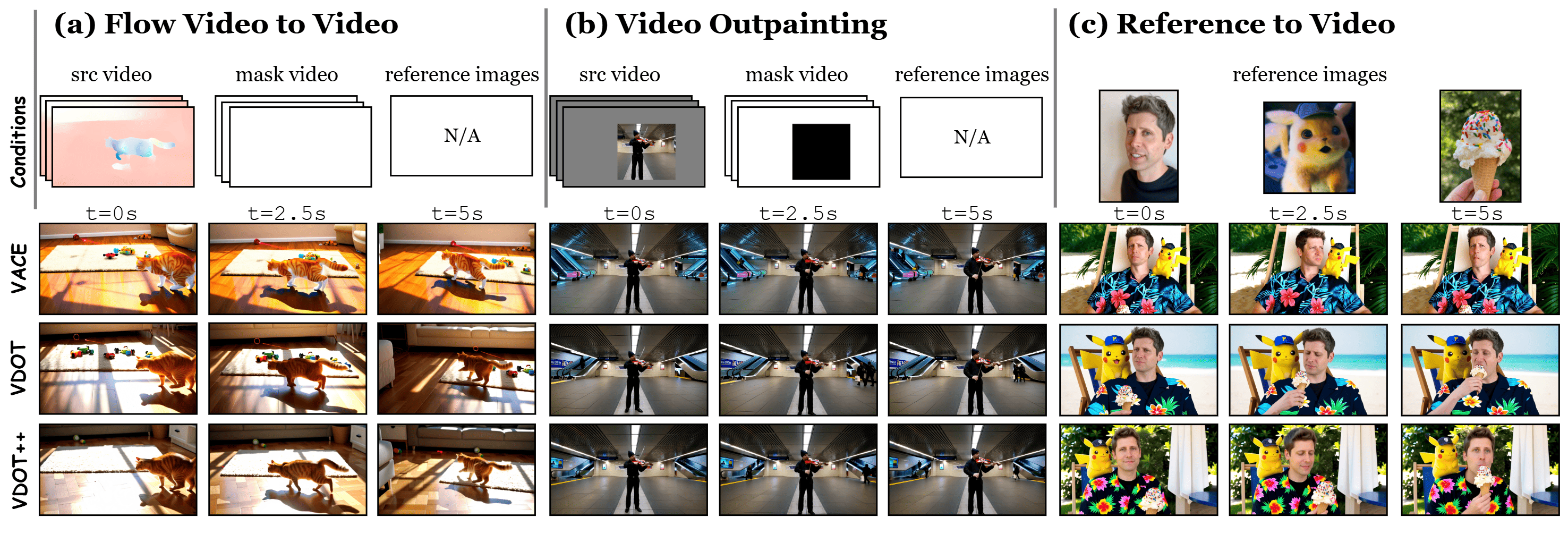}
  \caption{\textbf{Qualitative comparison on controllable video generation.}
  We compare the 100-NFE VACE teacher with 4-NFE \method\ and \methodpp\ on
  (a) flow-conditioned video-to-video generation, (b) video outpainting, and
  (c) multi-reference image-to-video generation. The top row shows the input
  conditions, and each method row shows frames at 0, 2.5, and 5 seconds.}
  \label{fig:visual1_vace}
\end{figure*}

\begin{figure*}[t]
  \centering
  \includegraphics[width=1.0\textwidth]{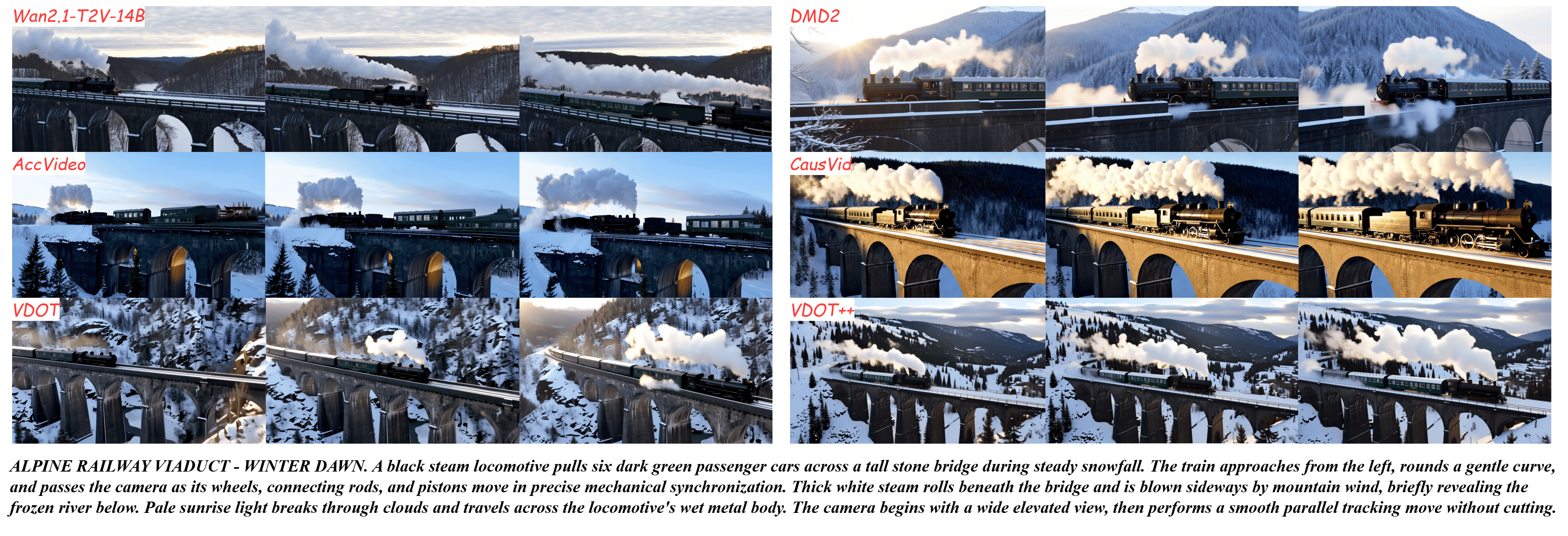}
  \caption{\textbf{Qualitative comparison on text-to-video generation.}
  Each row shows the first, middle, and last frames of the generated video.
  Zoom in for more details.}
  \label{fig:visual1_t2v}
\end{figure*}

\begin{figure*}[t]
  \centering
  \includegraphics[width=1.0\textwidth]{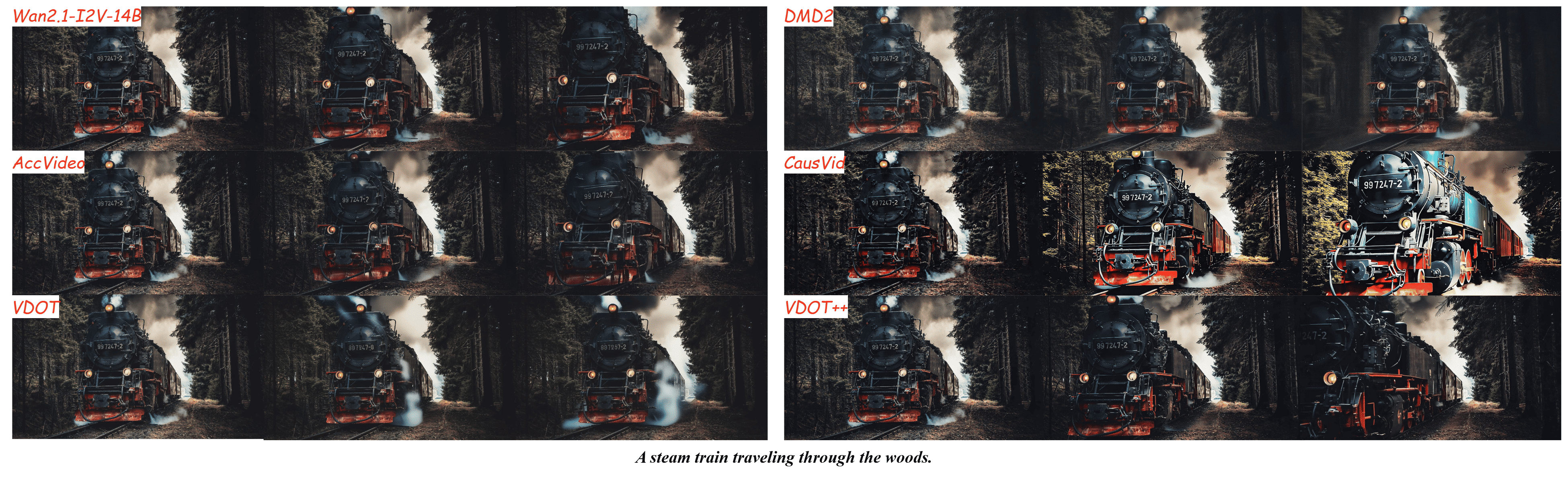}
  \caption{\textbf{Qualitative comparison on image-to-video generation.}
  Each row shows the first, middle, and last frames of the generated video.
  Zoom in for more details.}
  \label{fig:visual1_i2v}
\end{figure*}

\section{Experiments}
\label{sec:exp}

We evaluate \methodpp\ on three task families: text-to-video (T2V),
image-to-video (I2V), and condition-based (VACE-style) generation to answer
three questions. (Q1)~does the few-step \methodpp\ match or surpass many-step
baselines across all three families (Sec.~\ref{sec:main_results})?
(Q2)~does each proposed component contribute, and does the gain arise
precisely where the diagnosis of Sec.~\ref{sec:diagnosis} predicts
(Sec.~\ref{sec:ablation})? (Q3)~does the decoupled score-network design
enable cross-scale distillation (Sec.~\ref{sec:cross_scale})?

\subsection{Experimental Setup}
\label{sec:setup}

\noindent\textbf{Training setup.}
We distill VACE-Wan2.1-14B~\cite{jiang2025vace}, Wan2.1-T2V-14B, and
Wan2.1-I2V-14B~\cite{wan2025} into four-step generators with one recipe. By
default, the generator and both score networks use the 14B architecture. We
use the asymmetric $\ell_1$ objective and introduce adversarial training after
a warm-up. Appendix~\ref{sec:supp_impl} gives all optimization details. The
condition-based experiments use the multi-task corpus from our conference
version~\cite{wang2026vdot}, described in Appendix~\ref{sec:supp_data}.

\noindent\textbf{Benchmarks and baselines.}
For T2V and I2V we report VBench~\cite{huang2024vbench} and
VBench-I2V~\cite{huang2024vbench++}; for condition-based tasks we use the
VACE benchmark~\cite{jiang2025vace} and
\methodbench, introduced in our conference version~\cite{wang2026vdot}.
We report six VBench dimensions and their normalized average.
Table~\ref{tab:comp1} presents representative single-condition tasks; the
remaining results and protocols appear in Appendices~\ref{sec:supp_data}
and~\ref{sec:supp_uvcbench}. Tables~\ref{tab:comp1}--\ref{tab:i2v_main}
include many-step teachers and public models, same-backbone few-step
methods~\cite{yin2025causvid,zhang2025accvideo,zheng2025large,yin2024improved,wang2026vdot},
task-specific methods~\cite{chen2023control,zhang2023controlvideo,
ma2024follow,cong2023flatten,chen2024follow}, and commercial systems
~\cite{kling,vidu}. We implement DMD2 on Wan2.1-14B with the Self-Forcing
codebase~\cite{huang2025selfforcing}, as no official video implementation is
available.


\subsection{Main Results}
\label{sec:main_results}

With four denoising steps, \methodpp\ remains competitive with its many-step
teachers across all three task families. As conditioning weakens, its gains
concentrate on motion and imaging quality while temporal and subject
consistency remain stable. This task-dependent profile matches the motivation
for relaxing only unreliable correspondences.

\noindent\textbf{Unified evaluation on \methodbench.}
Table~\ref{tab:comp1} shows that \methodpp\ remains close to \method\ on the
dense pose, flow, and greyscale conditions, but improves the normalized average
on outpainting and R2V. In the latter tasks, observed content provides an
anchor while partial matching leaves unspecified regions and frames flexible.
The full results in Appendix~\ref{sec:supp_uvcbench} show a task-dependent
balance between appearance and dynamics: increased motion can accompany
lower imaging quality, with lower normalized averages on several dense and
composite tasks.

\noindent\textbf{Condition-based generation on the VACE benchmark.}
Table~\ref{tab:comp_vacebench} shows the same trend: gains concentrate on R2V,
outpainting, and extension, while structure-heavy tasks remain close to the
teacher. Figure~\ref{fig:visual1_vace} shows coherent motion under flow and
outpainting constraints. In multi-reference R2V, \methodpp\ combines all three
references while preserving their identities. These examples indicate that the
extra flexibility improves scene evolution without weakening adherence to the
provided controls.

\noindent\textbf{Text-to-video generation.}
On T2V (Table~\ref{tab:t2v_main}), \methodpp\ improves motion and imaging
quality while maintaining consistency, producing the best normalized average
among both few-step baselines and the many-step teacher. This agrees with the
method design: a text prompt fixes semantics but not trajectory, so partial
matching preserves plausible motion and the $\ell_1$ cost avoids averaging
distinct futures.
Figure~\ref{fig:visual1_t2v} shows clear motion without distorting the
locomotive or viaduct. Fine steam and snow details also remain stable.

\noindent\textbf{Image-to-video generation.}
I2V follows the same pattern (Table~\ref{tab:i2v_main}): the first frame anchors
appearance but not future motion, making partial matching preferable to copying
one arbitrary continuation. \methodpp\ exceeds the teacher and other few-step
methods in normalized average, with the largest gain appearing in motion rather
than static appearance preservation. Figure~\ref{fig:visual1_i2v} shows that
\methodpp\ preserves the locomotive and forest layout while generating forward
motion. It avoids the near-static behavior and appearance drift visible in
competing few-step results.

\subsection{Ablation Studies}
\label{sec:ablation}

We ablate the asymmetric partial-OT objective and adversarial training on VACE
and T2V, testing both their individual effects and the design choices within
each component.

\begin{table}[t]
  \centering
  \caption{Controlled objective factorial \emph{with adversarial
  refinement}: matching (balanced vs.\ partial) $\times$ cost ($\ell_2$
  vs.\ $\ell_1$), evaluated on the strongly-conditioned VACE tasks and the
  loosely-conditioned T2V (UVCBench and VBench normalized average).}
  \label{tab:factorial}
  \begin{tabular}{ll cc}
    \toprule
    Matching & Cost & VACE (strong)$\uparrow$ & T2V (loose)$\uparrow$ \\
    \midrule
    Balanced & $\ell_2$ & 78.31 & 77.62 \\
    Balanced & $\ell_1$ & 78.08 & 80.41 \\
    Partial  & $\ell_2$ & 78.11 &79.47  \\
    \cellcolor{tabhighlight}Partial & \cellcolor{tabhighlight}$\ell_1$ (ours)
      & \cellcolor{tabhighlight}\textbf{78.56} & \cellcolor{tabhighlight}\textbf{81.59} \\
    \bottomrule
  \end{tabular}
\end{table}

\begin{figure*}[t]
  \centering
  \includegraphics[width=\textwidth]{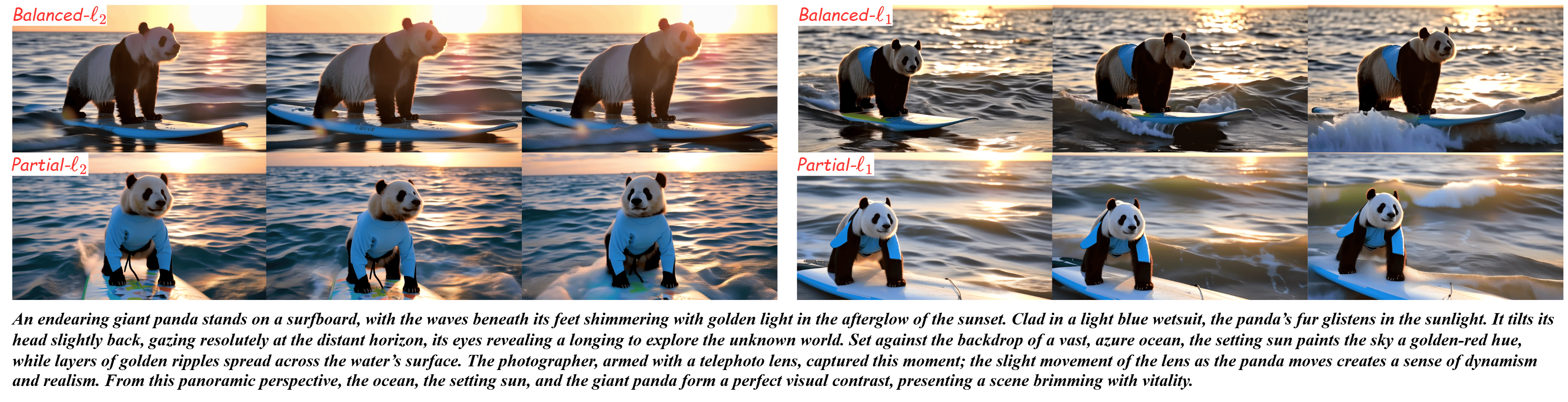}
  \caption{\textbf{Qualitative ablation of matching and ground cost on T2V.}
  Balanced and partial matching are compared with squared-$\ell_2$ and
  $\ell_1$ costs for the panda-on-a-surfboard prompt shown below the panels.
  Each variant shows three sampled frames from left to right. The example
  complements the controlled factorial in Table~\ref{tab:factorial}.}
  \label{fig:visual_aba_match}
\end{figure*}

\subsubsection{Asymmetric Partial-OT Objective}

\noindent\textbf{Controlled factorial: matching $\times$ cost $\times$
condition strength.}
Table~\ref{tab:factorial} varies matching ($\{$balanced, partial$\}$) and cost
($\{\ell_2,\ell_1\}$) while fixing adversarial refinement.
All variants remain close on strongly conditioned VACE. On T2V, partial
matching removes unreliable pairs, while $\ell_1$ avoids averaging distinct
modes; either change improves over balanced $\ell_2$, and their combination is
strongest. The contrast across task families supports the two-pathology
diagnosis: the proposed changes matter most when the condition admits diverse
valid outputs rather than acting as a uniform regularizer.

Figure~\ref{fig:visual_aba_match} illustrates the complementary effects in
a surfing example. Balanced $\ell_2$ shows little change in the panda's pose
across the sampled frames and does not clearly depict the requested blue
vest. Balanced $\ell_1$ produces more visible wave evolution, but the blue
garment is largely confined to the panda's back. Both partial variants make
the garment more apparent; partial $\ell_1$ also retains a clear torso and
surfboard outline as the surrounding waves change. This example complements
the aggregate T2V gains by showing how the combined objective can preserve
local appearance while allowing changes in pose and scene content.

\begin{table}[t]
  \centering
  \caption{Effect of marginal-relaxation direction on T2V under the
  controlled protocol of Table~\ref{tab:factorial}. $\tau_s$ and $\tau_t$ penalize deviations from the student
  and teacher marginals, respectively. Dynamic Degree measures motion
  freedom, while Imaging Quality measures visual degradation.}
  \label{tab:asym_tau}
  \footnotesize
  \setlength{\tabcolsep}{3pt}
  \begin{tabular}{l c ccc}
    \toprule
    Relaxation & $(\tau_s,\tau_t)$ & Dynamic$\uparrow$ & Imaging$\uparrow$ & Avg.$\uparrow$\\
    \midrule
    Balanced      & $(\infty,\infty)$ & 45.84 & 70.73 & 80.41\\
    Symmetric     & $(0.1,0.1)$         & 29.75 &\textbf{72.95} & 78.26\\
    Target relaxed & $(10,0.1)$         & 35.89 & 69.49 & 77.84\\
    \rowcolor{tabhighlight}
    Source relaxed (ours) & $(0.1,10)$  & \textbf{54.91} & 72.83 & \textbf{81.59}\\
    \bottomrule
  \end{tabular}
\end{table}

\begin{figure*}[t]
  \centering
  \includegraphics[width=\textwidth]{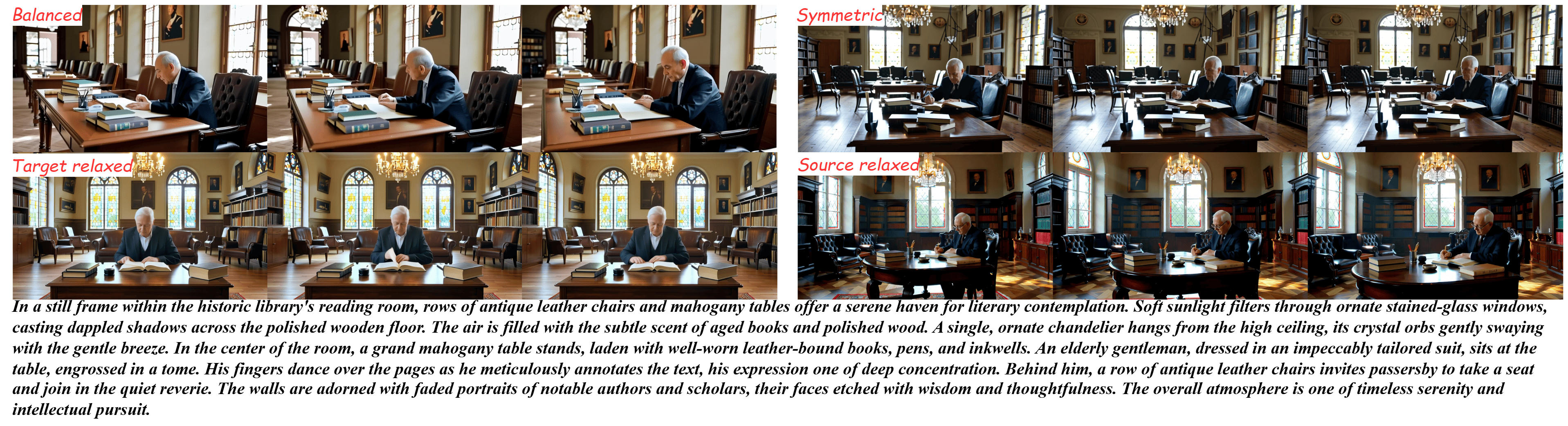}
  \caption{\textbf{Qualitative ablation of marginal-relaxation direction.}
  Three sampled frames are shown for each T2V variant under the library
  prompt. The settings are balanced OT, symmetric relaxation
  $(\tau_s,\tau_t)=(0.1,0.1)$, target-side relaxation $(10,0.1)$, and
  source-side relaxation $(0.1,10)$, corresponding to
  Table~\ref{tab:asym_tau}.}
  \label{fig:visual_aba_tau}
  \vspace{-0.15in}
\end{figure*}

\noindent\textbf{Asymmetric vs.\ symmetric relaxation.}
Table~\ref{tab:asym_tau} compares balanced, symmetric, and opposite assignments
of the two marginal penalties.
Source-side relaxation lets unreliable student tokens carry reduced mass
while retaining teacher-token coverage, yielding the highest Dynamic Degree
and normalized average. Its Imaging Quality (72.83) remains close to the
best value of 72.95 under symmetric relaxation. Symmetric and target-side relaxation instead underweight difficult
teacher tokens. Their lower averages, including the degradation from symmetric
relaxation relative to balanced OT, show that relaxing both marginals is not
sufficient; the direction of relaxation matters.

In the library example (Fig.~\ref{fig:visual_aba_tau}), source-side
relaxation preserves the seated reader, open book, and surrounding
library structure across the sampled frames, with modest changes in
framing and the reader's posture. The result illustrates
appearance preservation under relaxed student-side matching.
The higher Dynamic Degree in Table~\ref{tab:asym_tau} provides
the aggregate evidence for improved dynamics, which is less readily
assessed from these sparsely sampled frames.

\subsubsection{Training Strategy}

\begin{table}[t]
  \centering
  \caption{Effect of adversarial refinement and its update schedule on VACE
  and T2V. We report Imaging Quality and Aesthetic Quality, the two dimensions
  most directly targeted by adversarial refinement.}
  \label{tab:gan_mode}
  \footnotesize
  \setlength{\tabcolsep}{3pt}
  \begin{tabular}{l cc cc}
    \toprule
    & \multicolumn{2}{c}{VACE} & \multicolumn{2}{c}{T2V} \\
    \cmidrule(lr){2-3}\cmidrule(lr){4-5}
    Training & Imaging$\uparrow$ & Aesthetic$\uparrow$
             & Imaging$\uparrow$ & Aesthetic$\uparrow$ \\
    \midrule
    No GAN      & 64.56 & 61.99 & 69.80 & \textbf{70.25} \\
    Alternating & 65.59 & \textbf{62.71} & 69.85 & 68.44 \\
    \rowcolor{tabhighlight}
    Joint (ours) & \textbf{65.76} & 62.02 & \textbf{72.83} & 69.14 \\
    \bottomrule
  \end{tabular}
\end{table}

\noindent\textbf{Adversarial refinement and update schedule.}
Table~\ref{tab:gan_mode} compares training without a GAN, alternating updates,
and the joint objective of Sec.~\ref{sec:gradfuse}.
Adversarial refinement mainly improves imaging quality. Joint accumulation is
more stable than alternating updates because distribution matching preserves
structure while the discriminator refines local texture. Aesthetic Quality is
less consistent across schedules, supporting our use of the GAN as a weak
appearance-refinement signal rather than the primary training objective.




\subsection{Cross-Scale Distillation}
\label{sec:cross_scale}

\begin{table}[t]
  \centering
  \caption{Same-scale and cross-scale distillation on T2V, evaluated by the
  VBench normalized average.}
  \label{tab:cross_scale}
  \footnotesize
  \setlength{\tabcolsep}{3pt}
  \begin{tabular}{l c c c c}
    \toprule
    Method & Generator & Scores & \#NFE$\downarrow$ & VBench Avg.$\uparrow$ \\
    \midrule
    Wan2.1 (teacher)  & 14B  & --   & 50$\times$2 & \textbf{77.70} \\
    Wan2.1 (original) & 1.3B & --   & 50$\times$2 & 76.70 \\
    \midrule
    DMD2                   & 1.3B & 1.3B & 4 & 77.87 \\
    DMD2                   & 1.3B & 14B  & 4 & \textbf{78.15} \\
    \midrule
    \methodpp             & 1.3B & 1.3B & 4 & 79.98 \\
    \rowcolor{tabhighlight}
    \methodpp\ (Ours)     & 1.3B & 14B  & 4 & \textbf{80.51} \\
    \bottomrule
  \end{tabular}
\end{table}

\begin{figure*}[t]
  \centering
  \includegraphics[width=\textwidth]{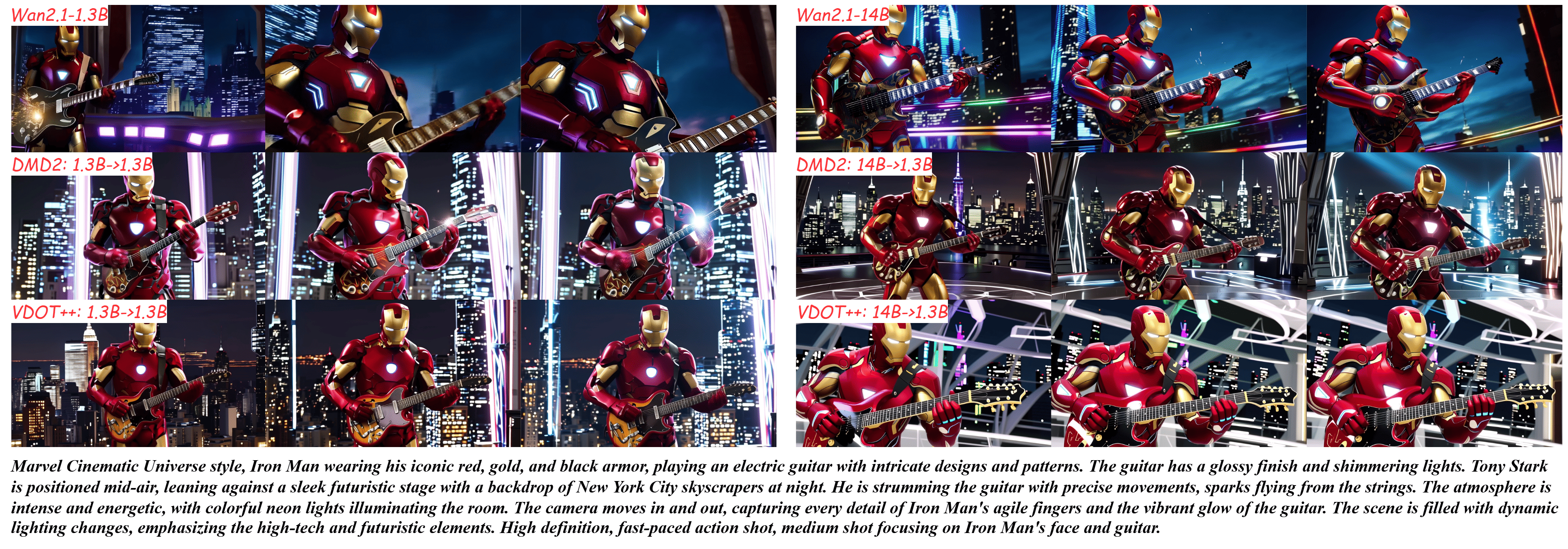}
  \caption{\textbf{Qualitative comparison of same-scale and cross-scale
  distillation on T2V.} The top row shows the original Wan2.1 models; the
  middle and bottom rows show the DMD2 baseline and \methodpp, respectively.
  In distilled variants, $A\!\rightarrow\!B$ denotes the score-network and
  generator sizes: all generators are 1.3B, with 1.3B scores on the left and
  14B scores on the right. Each variant shows three sampled frames for the
  guitar-playing prompt below the panels.}
  \label{fig:visual_aba_scale}
\end{figure*}

Because score networks are discarded after training, they can exceed the
deployed generator in size. We fix a Wan2.1-T2V-1.3B generator and compare
1.3B and 14B score networks for DMD2 and \methodpp.
As shown in Table~\ref{tab:cross_scale}, larger score networks improve the
normalized average from 77.87 to 78.15 for DMD2 and from 79.98 to 80.51 for
\methodpp. The larger gain for \methodpp\ is consistent with its use of
teacher information in both score matching and transport-based supervision.
With 14B score networks, the 1.3B \methodpp\ generator surpasses the
14B many-step teacher in normalized average. The larger networks are used only
during training; deployment remains a four-step 1.3B generator, so inference
cost is unchanged.

Figure~\ref{fig:visual_aba_scale} complements this comparison with a
guitar-playing scene. Both \methodpp\ variants retain the red-and-gold armor
and guitar across the sampled frames. With 14B score networks, the fretboard,
armor boundaries, and surrounding stage beams are clearly delineated as the
performer's pose changes. The example shows that a 1.3B generator can retain
fine foreground and background structure under cross-scale supervision;
the aggregate comparison in Table~\ref{tab:cross_scale}, rather than this
single scene, establishes the benefit of increasing score-network capacity.

\section{Conclusion}
\label{sec:conclusion}

This paper extends few-step OTD across T2V, I2V, and condition-based video
creation by accounting for their different output ambiguity. \methodpp\ uses
asymmetric unbalanced transport to suppress unreliable student
correspondences while retaining teacher coverage, and an $\ell_1$ cost to
avoid mean-based aggregation of incompatible targets. Experiments and
controlled ablations support this conditioning-dependent design: it behaves
similarly to balanced OTD under dense control, but better preserves valid
motion and detail as the output space widens. Joint adversarial refinement
improves local appearance, while cross-scale distillation uses larger auxiliary
score networks to strengthen a compact generator.

\noindent\textbf{Limitations and future work.}
The current method solves spatial-token OT independently within corresponding
temporal-frame pairs. It therefore models fine spatial correspondence but does
not transport tokens across frames or jointly optimize a spatiotemporal plan.
Spatiotemporal or trajectory-aware transport could capture motion-level
correspondence, while adaptive marginal penalties could tune relaxation from
condition strength or transport confidence. Extending the current four-step
bidirectional setting to one-step, streaming, or autoregressive generation
also requires handling stronger rollout errors and support mismatch.

\bibliographystyle{IEEEtran}
\bibliography{references}

\clearpage
\appendices

\section{Notation}
\label{sec:supp_notation}

Table~\ref{tab:notation} consolidates the symbols defined throughout the
main paper.

\begin{table}[t]
  \centering
  \caption{Main notation used throughout the paper.}
  \label{tab:notation}
  \footnotesize
  \setlength{\tabcolsep}{3pt}
  \begin{tabular}{@{}l p{0.60\columnwidth}@{}}
    \toprule
    Symbol & Meaning \\
    \midrule
    \multicolumn{2}{l}{\emph{Networks and generation}} \\
    $G_\theta$ & few-step generator (deployed at inference) \\
    $F_\psi$ / $F_\phi$ & frozen teacher (real) / trainable fake score network \\
    $D_\omega$ & discriminator head embedded in $F_\phi$ \\
    $\bm{z}$, $\bm{\xi}$ & generator-input noise; re-noising noise \\
    $\bm{c}$ & heterogeneous generation condition \\
    $\hat{\bm{x}}_0$ & generator output $G_\theta(\bm{z},\bm{c})$ \\
    $t$, $(\alpha_t,\sigma_t)$ & diffusion time; noise schedule \\
    $\bm{x}_t$ & re-noised output $\alpha_t\hat{\bm{x}}_0+\sigma_t\bm{\xi}$ \\
    \midrule
    \multicolumn{2}{l}{\emph{Optimal transport}} \\
    $\pfake$, $\preal$ & student (fake) / teacher (real) video distributions \\
    $q$ & temporal-frame index \\
    $\bm{A}_{t,q}$, $\bm{B}_{t,q}$ & student / teacher spatial-token sets
      $[\bm{a}_{t,q,i}]$ / $[\bm{b}_{t,q,j}]$ in frame $q$ at diffusion time $t$ \\
    $\bm{u}$, $\bm{\mu}$ & marginal weights in $\Delta^{I-1}$ / $\Delta^{J-1}$ (uniform) \\
    $d(\cdot,\cdot)$, $\Cmat=[C_{ij}]$ & ground metric; cost matrix $C_{ij}=d(\bm{a}_i,\bm{b}_j)$ \\
    $\Tplan_{t,q}$, $\Topt_{t,q}$ & within-frame plan between $\bm{A}_{t,q}$ and $\bm{B}_{t,q}$; optimum \\
    $\bm{r}_{t,q}$, $\bm{c}_{t,q}$ & row / column projections
      $\Topt_{t,q}\bm{1}_J$ / $(\Topt_{t,q})^{\!\top}\bm{1}_I$ \\
    $\epsilon$ & entropic regularization strength \\
    $\tau_s$, $\tau_t$ & student / teacher marginal-penalty strengths \\
    $\rho_s$, $\rho_t$ & Sinkhorn scaling factors $\tau/(\tau+\epsilon)$ \\
    $\bm{f}$, $\bm{g}$, $\bm{K}$ & dual potentials; Gibbs kernel $e^{-\Cmat/\epsilon}$ \\
    $\kappa$ & total optimized mass $\sum_{ij}T^*_{ij}$ \\
    \midrule
    \multicolumn{2}{l}{\emph{Training}} \\
    $w_{\mathrm{DMD}},w_{\mathrm{OT}}$ & distribution-matching loss weights \\
    $w_{\mathrm{GAN\text{-}G}},w_{\mathrm{GAN\text{-}D}}$ & generator / discriminator adversarial weights \\
    $R$ & critic-update ratio \\
    $\sg(\cdot)$, $\softplus(\cdot)$ & stop-gradient; softplus $\log(1+e^{x})$ \\
    \bottomrule
  \end{tabular}
\end{table}

\section{Proofs}
\label{sec:supp_proofs}

\noindent\textbf{Proof of Proposition~\ref{prop:envelope}.}
$W$ is a pointwise infimum of functions affine in $\Cmat$, hence concave.
The strictly convex entropy makes $\Topt$ unique, and Danskin's theorem
yields the gradient~\cite{peyre2019computational}.

\noindent\textbf{Proof of Proposition~\ref{prop:mean}.}
Factoring the gradient gives
\begin{equation}
\sum_j T^*_{ij}(\bm{a}_i-\bm{b}_j)
=\Big(\sum_jT^*_{ij}\Big)\bm{a}_i-\sum_jT^*_{ij}\bm{b}_j
=r_i(\bm{a}_i-\bar{\bm{b}}_i).
\end{equation}
Because $r_i>0$, it vanishes exactly at
$\bm{a}_i=\bar{\bm{b}}_i$.

\noindent\textbf{Proof of Proposition~\ref{prop:limit}.}
Balanced-feasible plans incur zero penalty, so the value $V(\tau)$
of~\eqref{eq:uot} is non-decreasing in $\tau$ and bounded above by $\Weps$.
This forces the KL penalties of $\Tplan_\tau$ to vanish as
$\tau\to\infty$. Every accumulation point of the bounded family
$\{\Tplan_\tau\}$ is therefore balanced-feasible with objective at most
$\Weps$, and hence optimal for~\eqref{eq:eot}. Strict convexity of the
entropy gives uniqueness, convergence of the whole family, and
$V(\tau)\nearrow\Weps$.

\noindent\textbf{Proof of Proposition~\ref{prop:median}.}
$\varphi$ is convex and piecewise linear. The optimality condition
$0\in\partial\varphi(a)$ is precisely the weighted-median condition.

\section{Implementation Details}
\label{sec:supp_impl}

\begin{table*}[t]
\caption{
\textbf{Additional single-condition results on \methodbench}. We bold the
\textbf{best results} and underline the \underline{second-best results}.
}

\centering
\footnotesize
\setlength{\tabcolsep}{5.8pt}
\begin{tabular}{l|l|c|c|ccccccc}

\toprule

Type & Method & Base Model & \#NFE$\downarrow$ & \multicolumn{7}{c}{\textbf{Video Quality \& Video Consistency}} 
 \\

& & & & \rotatebox{90}{\begin{minipage}{1.8cm} {\tikz\fill[video_quality] (0,0) circle (.5ex);} \centering Aesthetic Quality\end{minipage}}
& \rotatebox{90}{\begin{minipage}{1.8cm} {\tikz\fill[video_quality] (0,0) circle (.5ex);} \centering Background Consistency\end{minipage}}
& \rotatebox{90}{\begin{minipage}{1.8cm} {\tikz\fill[video_quality] (0,0) circle (.5ex);} \centering Dynamic Degree\end{minipage}}
& \rotatebox{90}{\begin{minipage}{1.8cm} {\tikz\fill[video_quality] (0,0) circle (.5ex);} \centering Imaging Quality\end{minipage}}
& \rotatebox{90}{\begin{minipage}{1.8cm} {\tikz\fill[video_quality] (0,0) circle (.5ex);} \centering Motion Smoothness\end{minipage}}
& \rotatebox{90}{\begin{minipage}{1.8cm} {\tikz\fill[video_quality] (0,0) circle (.5ex);} \centering Subject Consistency\end{minipage}}
& \rotatebox{90}{\begin{minipage}{1.8cm} {\tikz\fill[score_average] (0,0) circle (.5ex);} \centering Normalized Average\end{minipage}}
 \\

\midrule 

\multirow{6}{*}{Depth}
& Control-A-Video~\cite{chen2023control}  &SD-1.5  &100 &57.22\%  &92.89\%  &20.00\%  &66.95\%  &98.06\%  &91.85\%  &71.16\%   \\ 
& ControlVideo~\cite{zhang2023controlvideo}  &SD-1.5  &100  &\textbf{64.50\%}  &97.99\%  &5.00\%  &\textbf{71.76\%}  &98.10\%  &97.41\%  &72.46\%   \\ 
& VACE~\cite{jiang2025vace} &LTX-0.9B  &80  &58.62\%  &\textbf{98.18\%}  &30.00\%  &70.14\%  &\underline{99.29\%}  &97.81\%  &75.67\%   \\
& VACE~\cite{jiang2025vace} &Wan-14B  &100  &\underline{64.36\%}  &97.63\%  &\underline{35.00\%}  &70.29\%  &99.17\%  &97.58\%  &77.34\%   \\
& \cellcolor{tabhighlight}\methodbf (Ours)  &\cellcolor{tabhighlight}Wan-14B  &\cellcolor{tabhighlight}4  &\cellcolor{tabhighlight}64.28\%  &\cellcolor{tabhighlight}\textbf{98.18\%}  &\cellcolor{tabhighlight}\textbf{40.00\%}  &\cellcolor{tabhighlight}\underline{71.24\%}  &\cellcolor{tabhighlight}\textbf{99.35\%}  &\cellcolor{tabhighlight}\textbf{97.96\%}  &\cellcolor{tabhighlight}\textbf{78.50\%}   \\

& \cellcolor{tabhighlight}\methodppbf (Ours)  &\cellcolor{tabhighlight}Wan-14B  &\cellcolor{tabhighlight}4  &\cellcolor{tabhighlight}63.04\%  &\cellcolor{tabhighlight}97.75\%  &\cellcolor{tabhighlight}\textbf{40.00\%} &\cellcolor{tabhighlight}71.06\%  &\cellcolor{tabhighlight}\underline{99.29\%}  &\cellcolor{tabhighlight}\underline{97.85\%}  &\cellcolor{tabhighlight}\underline{78.17\%}   \\
\midrule

\multirow{5}{*}{Scribble} 
& ControlVideo~\cite{zhang2023controlvideo} &SD-1.5  &100  &\underline{54.79\%}  &96.19\%  &5.00\%  &\underline{64.30\%}  &98.53\%  &94.64\%  &68.91\%   \\
& VACE~\cite{jiang2025vace} &LTX-0.9B  &80  &44.69\%  &96.88\%  &40.00\%  &61.06\%  &\textbf{99.17\%}  &93.51\%  &72.55\%  \\
& VACE~\cite{jiang2025vace} &Wan-14B  &100  &\textbf{55.70\%}  &\underline{96.91\%}  &\textbf{55.00\%}  &58.69\%  &\underline{98.99\%}  &94.51\%  &\underline{76.63\%}   \\
& \cellcolor{tabhighlight}\methodbf (Ours)  &\cellcolor{tabhighlight}Wan-14B  &\cellcolor{tabhighlight}4  &\cellcolor{tabhighlight}54.24\%  &\cellcolor{tabhighlight}\textbf{97.13\%}  &\cellcolor{tabhighlight}\underline{50.00\%}  &\cellcolor{tabhighlight}\textbf{65.19\% } &\cellcolor{tabhighlight}98.90\%  &\cellcolor{tabhighlight}\textbf{95.20\%}  &\cellcolor{tabhighlight}\textbf{76.77\%}   \\

& \cellcolor{tabhighlight}\methodppbf (Ours)  &\cellcolor{tabhighlight}Wan-14B  &\cellcolor{tabhighlight}4  &\cellcolor{tabhighlight}53.86\%  &\cellcolor{tabhighlight}96.59\%  &\cellcolor{tabhighlight}\underline{50.00\%} &\cellcolor{tabhighlight}62.32\%  &\cellcolor{tabhighlight}98.92\%  &\cellcolor{tabhighlight}\underline{94.70\%}  &\cellcolor{tabhighlight}76.06\%   \\

\midrule

\multirow{4}{*}{Extension} 
& VACE~\cite{jiang2025vace} &LTX-0.9B  &80  &51.54\%  &\textbf{96.48\%}  &20.00\%  &62.22\%  &\textbf{99.21\%}  &\textbf{94.67\%}  &70.68\%  \\
 & VACE~\cite{jiang2025vace} &Wan-14B  &100  &57.52\%  &\underline{95.46\%}  &55.00\%  &63.36\%  &98.99\%  &\underline{92.27\%}  &77.10\% \\
& \cellcolor{tabhighlight}\methodbf (Ours)  &\cellcolor{tabhighlight}Wan-14B  &\cellcolor{tabhighlight}4  &\cellcolor{tabhighlight}\textbf{58.93\%}  &\cellcolor{tabhighlight}94.18\%  &\cellcolor{tabhighlight}\textbf{75.00\%}  &\cellcolor{tabhighlight}\underline{65.11\%} &\cellcolor{tabhighlight}\textbf{99.21\%}  &\cellcolor{tabhighlight}91.09\%  &\cellcolor{tabhighlight}\textbf{80.59\%}  \\

& \cellcolor{tabhighlight}\methodppbf (Ours)  &\cellcolor{tabhighlight}Wan-14B  &\cellcolor{tabhighlight}4  &\cellcolor{tabhighlight}\underline{58.68\%}  &\cellcolor{tabhighlight}94.51\%  &\cellcolor{tabhighlight}\underline{72.50\%} &\cellcolor{tabhighlight}\textbf{65.71\%}  &\cellcolor{tabhighlight}98.75\%  &\cellcolor{tabhighlight}91.83\%  &\cellcolor{tabhighlight}\underline{80.33\%}   \\


\bottomrule
\end{tabular}
\label{tab:uvcbench_single_supp}
\end{table*}

\begin{table*}[t]
\caption{
\textbf{Quantitative evaluations of composite tasks on \methodbench}. We compare the automated score metrics of our \method with VACE on the dimensions of video quality and video consistency.
}

\centering
\footnotesize
\setlength{\tabcolsep}{4.5pt}
\begin{tabular}{l|l|l|c|c|cccccc|c}

\toprule

Condition1 & Condition2 & Method & Base Model & \#NFE$\downarrow$ & \multicolumn{7}{c}{\textbf{Video Quality \& Video Consistency}} \\

& & & & & \rotatebox{90}{\begin{minipage}{1.8cm} {\tikz\fill[video_quality] (0,0) circle (.5ex);} \centering Aesthetic Quality\end{minipage}}
& \rotatebox{90}{\begin{minipage}{1.8cm} {\tikz\fill[video_quality] (0,0) circle (.5ex);} \centering Background Consistency\end{minipage}}
& \rotatebox{90}{\begin{minipage}{1.8cm} {\tikz\fill[video_quality] (0,0) circle (.5ex);} \centering Dynamic Degree\end{minipage}}
& \rotatebox{90}{\begin{minipage}{1.8cm} {\tikz\fill[video_quality] (0,0) circle (.5ex);} \centering Imaging Quality\end{minipage}}
& \rotatebox{90}{\begin{minipage}{1.8cm} {\tikz\fill[video_quality] (0,0) circle (.5ex);} \centering Motion Smoothness\end{minipage}}
& \rotatebox{90}{\begin{minipage}{1.8cm} {\tikz\fill[video_quality] (0,0) circle (.5ex);} \centering Subject Consistency\end{minipage}}
& \rotatebox{90}{\begin{minipage}{1.8cm} {\tikz\fill[score_average] (0,0) circle (.5ex);} \centering Normalized Average\end{minipage}} \\

\midrule 

\multirow{20}{*}{FirstFrame}& \multirow{4}{*}{Depth}
& VACE~\cite{jiang2025vace} &LTX-0.9B  &80  &64.78\%  &\textbf{96.90\%}  &40.00\%  &66.19\%  &98.96\%  &94.49\%  &76.89\%  \\
& & VACE~\cite{jiang2025vace} &Wan-14B  &100  &\textbf{68.28\%}  &96.74\%  &40.00\%  &68.43\%  &98.92\%  &94.38\%  &77.79\%  \\
& &\cellcolor{tabhighlight}\methodbf (Ours) & \cellcolor{tabhighlight}Wan-14B  & \cellcolor{tabhighlight}4  &\cellcolor{tabhighlight}67.34\% &\cellcolor{tabhighlight}96.67\% &\cellcolor{tabhighlight}\textbf{45.00\%} &\cellcolor{tabhighlight}\textbf{70.09\%} &\cellcolor{tabhighlight}98.97\% &\cellcolor{tabhighlight}\textbf{94.64\%} &\cellcolor{tabhighlight}\textbf{78.78\%}  \\
& &\cellcolor{tabhighlight}\methodppbf (Ours) & \cellcolor{tabhighlight}Wan-14B  & \cellcolor{tabhighlight}4  
&\cellcolor{tabhighlight}67.68\% 
&\cellcolor{tabhighlight}96.75\% 
&\cellcolor{tabhighlight}40.00\%
&\cellcolor{tabhighlight}67.35\%
&\cellcolor{tabhighlight}\textbf{99.05\%} &\cellcolor{tabhighlight}94.40\% &\cellcolor{tabhighlight}77.54\%  \\

\cline{2-12} 

& \multirow{4}{*}{Pose}
& VACE~\cite{jiang2025vace} &LTX-0.9B  &80  &64.75\%  &\textbf{96.29\%}  &15.00\%  &63.90\%  &\textbf{99.13\%}  &94.63\%  &72.28\%  \\
& & VACE~\cite{jiang2025vace} &Wan-14B  &100  &66.66\%  &95.91\%  &\textbf{40.00\%}  &66.99\%  &98.94\%  &\textbf{94.70\%}  &\textbf{77.19\%}   \\
& &\cellcolor{tabhighlight}\methodbf (Ours) & \cellcolor{tabhighlight}Wan-14B  & \cellcolor{tabhighlight}4  &\cellcolor{tabhighlight}66.26\% &\cellcolor{tabhighlight}96.20\% &\cellcolor{tabhighlight}30.00\% &\cellcolor{tabhighlight}\textbf{68.63\%} &\cellcolor{tabhighlight}99.08\% &\cellcolor{tabhighlight}94.22\% &\cellcolor{tabhighlight}75.73\%  \\
& &\cellcolor{tabhighlight}\methodppbf (Ours) & \cellcolor{tabhighlight}Wan-14B  & \cellcolor{tabhighlight}4  
&\cellcolor{tabhighlight}\textbf{66.78\%} 
&\cellcolor{tabhighlight}96.01\% 
&\cellcolor{tabhighlight}\textbf{40.00\%}
&\cellcolor{tabhighlight}66.76\%
&\cellcolor{tabhighlight}99.06\% &\cellcolor{tabhighlight}94.18\% &\cellcolor{tabhighlight}77.13\%  \\

\cline{2-12} 

& \multirow{4}{*}{Flow}
& VACE~\cite{jiang2025vace} &LTX-0.9B  &80  &54.38\%  &\textbf{95.19\%}  &70.00\%  &52.16\%  &\textbf{98.21\%}  &91.05\%  &76.83\%  \\
& & VACE~\cite{jiang2025vace} &Wan-14B  &100  &59.80\%  &94.66\%  &85.00\%  &53.41\%  &97.80\%  &\textbf{91.62\%}  &80.38\%  \\
& &\cellcolor{tabhighlight}\methodbf (Ours) & \cellcolor{tabhighlight}Wan-14B  & \cellcolor{tabhighlight}4  &\cellcolor{tabhighlight}59.27\% &\cellcolor{tabhighlight}95.01\% &\cellcolor{tabhighlight}85.00\% &\cellcolor{tabhighlight}\textbf{55.75\%} &\cellcolor{tabhighlight}97.92\% &\cellcolor{tabhighlight}91.58\% &\cellcolor{tabhighlight}80.75\%  \\
& &\cellcolor{tabhighlight}\methodppbf (Ours) & \cellcolor{tabhighlight}Wan-14B  & \cellcolor{tabhighlight}4  
&\cellcolor{tabhighlight}\textbf{60.24\%} 
&\cellcolor{tabhighlight}94.59\% 
&\cellcolor{tabhighlight}\textbf{90.00\%}
&\cellcolor{tabhighlight}52.75\%
&\cellcolor{tabhighlight}97.89\% &\cellcolor{tabhighlight}91.08\% &\cellcolor{tabhighlight}\textbf{81.09\%}  \\

\cline{2-12} 

& \multirow{4}{*}{Scribble}
& VACE~\cite{jiang2025vace} &LTX-0.9B  &80  &55.16\%  &96.57\%  &35.00\%  &66.52\%  &99.05\%  &94.36\%  &74.44\%  \\
& & VACE~\cite{jiang2025vace} &Wan-14B  &100  &58.26\%  &96.78\%  &\textbf{50.00\%}  &66.91\%  &98.91\%  &94.89\%  & 77.63\%  \\
& &\cellcolor{tabhighlight}\methodbf (Ours) & \cellcolor{tabhighlight}Wan-14B  & \cellcolor{tabhighlight}4  &\cellcolor{tabhighlight}57.54\% &\cellcolor{tabhighlight}96.64\% &\cellcolor{tabhighlight}\textbf{50.00\%} &\cellcolor{tabhighlight}\textbf{68.23\%} &\cellcolor{tabhighlight}98.97\% &\cellcolor{tabhighlight}94.94\% &\cellcolor{tabhighlight}\textbf{77.72\%}  \\
& &\cellcolor{tabhighlight}\methodppbf (Ours) & \cellcolor{tabhighlight}Wan-14B  & \cellcolor{tabhighlight}4  
&\cellcolor{tabhighlight}\textbf{58.33\%} 
&\cellcolor{tabhighlight}\textbf{96.91\%} 
&\cellcolor{tabhighlight}45.00\%
&\cellcolor{tabhighlight}66.91\%
&\cellcolor{tabhighlight}\textbf{99.06\%} &\cellcolor{tabhighlight}\textbf{95.04\%} &\cellcolor{tabhighlight}76.88\%  \\

\cline{2-12} 

& \multirow{4}{*}{Grey}
& VACE~\cite{jiang2025vace} &LTX-0.9B  &80  &66.45\%  &97.97\%  &10.00\%  &63.60\%  &99.34\%  &97.85\%  &72.54\%  \\
& & VACE~\cite{jiang2025vace} &Wan-14B  &100  &68.87\%  &\textbf{98.25\%}  &15.00\%  &\textbf{66.11\%}  &99.30\%  &98.02\%  & 74.25\%  \\
& &\cellcolor{tabhighlight}\methodbf (Ours) & \cellcolor{tabhighlight}Wan-14B  & \cellcolor{tabhighlight}4  &\cellcolor{tabhighlight}\textbf{69.08\%} &\cellcolor{tabhighlight}98.01\% &\cellcolor{tabhighlight}\textbf{20.00\%} &\cellcolor{tabhighlight}65.92\% &\cellcolor{tabhighlight}99.30\% &\cellcolor{tabhighlight}98.08\%  &\cellcolor{tabhighlight}\textbf{75.07\% } \\
& &\cellcolor{tabhighlight}\methodppbf (Ours) & \cellcolor{tabhighlight}Wan-14B  & \cellcolor{tabhighlight}4  
&\cellcolor{tabhighlight}68.99\% 
&\cellcolor{tabhighlight}\textbf{98.25\%} 
&\cellcolor{tabhighlight}\textbf{20.00\%}
&\cellcolor{tabhighlight}65.11\%
&\cellcolor{tabhighlight}\textbf{99.36\%} &\cellcolor{tabhighlight}\textbf{98.16\%} &\cellcolor{tabhighlight}74.98\%  \\

\midrule 

\multirow{20}{*}{Reference}& \multirow{4}{*}{Depth}
& VACE~\cite{jiang2025vace} &LTX-0.9B  &80  &52.92\%  &87.63\%  &60.00\%  &63.22\%  &\textbf{98.82\%}  &77.83\%  &73.40\%  \\
& & VACE~\cite{jiang2025vace} &Wan-14B  &100  &70.07\%  &96.53\%  &\textbf{65.00\%}  &67.61\%  &98.39\%  &94.97\%  & \textbf{82.09\%}  \\
& &\cellcolor{tabhighlight}\methodbf (Ours) & \cellcolor{tabhighlight}Wan-14B  & \cellcolor{tabhighlight}4  &\cellcolor{tabhighlight}69.49\% &\cellcolor{tabhighlight}96.64\% &\cellcolor{tabhighlight}\textbf{65.00\%} &\cellcolor{tabhighlight}\textbf{67.85\%} &\cellcolor{tabhighlight}98.44\% &\cellcolor{tabhighlight}95.05\% &\cellcolor{tabhighlight}82.08\%  \\
& &\cellcolor{tabhighlight}\methodppbf (Ours) & \cellcolor{tabhighlight}Wan-14B  & \cellcolor{tabhighlight}4  
&\cellcolor{tabhighlight}\textbf{70.44\%} 
&\cellcolor{tabhighlight}\textbf{96.71\%} 
&\cellcolor{tabhighlight}60.00\%
&\cellcolor{tabhighlight}66.00\%
&\cellcolor{tabhighlight}98.71\% &\cellcolor{tabhighlight}\textbf{95.12\%} &\cellcolor{tabhighlight}81.33\%  \\

\cline{2-12} 

& \multirow{4}{*}{Pose}
& VACE~\cite{jiang2025vace} &LTX-0.9B  &80  &49.89\%  &85.24\%  &65.00\%  &52.74\%  &\textbf{98.43\%}  &71.46\%  &70.46\%  \\
& & VACE~\cite{jiang2025vace} &Wan-14B  &100  &\textbf{71.05\%}  &95.38\%  &\textbf{75.00\%}  &67.72\%  &97.77\%  &93.40\%  & 83.38\%  \\
& &\cellcolor{tabhighlight}\methodbf (Ours) & \cellcolor{tabhighlight}Wan-14B  & \cellcolor{tabhighlight}4  &\cellcolor{tabhighlight}70.81\% &\cellcolor{tabhighlight}95.97\% &\cellcolor{tabhighlight}\textbf{75.00\%} &\cellcolor{tabhighlight}\textbf{67.81\%} &\cellcolor{tabhighlight}98.19\% &\cellcolor{tabhighlight}94.24\% &\cellcolor{tabhighlight}\textbf{83.67\%}  \\
& &\cellcolor{tabhighlight}\methodppbf (Ours) & \cellcolor{tabhighlight}Wan-14B  & \cellcolor{tabhighlight}4  
&\cellcolor{tabhighlight}70.77\% 
&\cellcolor{tabhighlight}\textbf{96.23\%} 
&\cellcolor{tabhighlight}\textbf{75.00\%}
&\cellcolor{tabhighlight}66.84\%
&\cellcolor{tabhighlight}98.33\% &\cellcolor{tabhighlight}\textbf{94.48\%} &\cellcolor{tabhighlight}83.61\%  \\

\cline{2-12} 

& \multirow{4}{*}{Flow}
& VACE~\cite{jiang2025vace} &LTX-0.9B  &80  &52.33\%  &83.37\%  &85.00\%  &58.50\%  &\textbf{98.36\%}  &66.36\%  &73.98\%  \\
& & VACE~\cite{jiang2025vace} &Wan-14B  &100  &65.92\%  &96.30\%  &\textbf{90.00\%}  &\textbf{66.74\%}  &97.71\%  &93.01\%  & \textbf{84.94\%}  \\
& &\cellcolor{tabhighlight}\methodbf (Ours) & \cellcolor{tabhighlight}Wan-14B  & \cellcolor{tabhighlight}4  &\cellcolor{tabhighlight}65.31\% &\cellcolor{tabhighlight}96.67\% &\cellcolor{tabhighlight}80.00\% &\cellcolor{tabhighlight}66.54\% &\cellcolor{tabhighlight}98.00\% &\cellcolor{tabhighlight}93.78\% &\cellcolor{tabhighlight}83.38\%  \\
& &\cellcolor{tabhighlight}\methodppbf (Ours) & \cellcolor{tabhighlight}Wan-14B  & \cellcolor{tabhighlight}4  
&\cellcolor{tabhighlight}\textbf{65.93\%} 
&\cellcolor{tabhighlight}\textbf{96.83\%} 
&\cellcolor{tabhighlight}85.00\%
&\cellcolor{tabhighlight}64.38\%
&\cellcolor{tabhighlight}98.00\% &\cellcolor{tabhighlight}\textbf{94.01\%} &\cellcolor{tabhighlight}84.02\%  \\

\cline{2-12} 

& \multirow{4}{*}{Scribble}
& VACE~\cite{jiang2025vace} &LTX-0.9B  &80  &51.64\%  &89.42\%  &50.00\%  &56.96\%  &98.70\%  &73.25\%  &69.99\%  \\
& & VACE~\cite{jiang2025vace} &Wan-14B  &100  &65.57\%  &96.43\%  &\textbf{75.00\%}  &\textbf{66.20\%}  &98.19\%  &93.18\%  & \textbf{82.42\%}  \\
& &\cellcolor{tabhighlight}\methodbf (Ours) & \cellcolor{tabhighlight}Wan-14B  & \cellcolor{tabhighlight}4  &\cellcolor{tabhighlight}64.10\% &\cellcolor{tabhighlight}96.78\% &\cellcolor{tabhighlight}\textbf{75.00\%} &\cellcolor{tabhighlight}65.35\% &\cellcolor{tabhighlight}98.30\% &\cellcolor{tabhighlight}93.41\% &\cellcolor{tabhighlight}82.16\%  \\
& &\cellcolor{tabhighlight}\methodppbf (Ours) & \cellcolor{tabhighlight}Wan-14B  & \cellcolor{tabhighlight}4  
&\cellcolor{tabhighlight}\textbf{66.23\%} 
&\cellcolor{tabhighlight}\textbf{97.17\%} 
&\cellcolor{tabhighlight}70.00\%
&\cellcolor{tabhighlight}63.09\%
&\cellcolor{tabhighlight}\textbf{98.81\%} &\cellcolor{tabhighlight}\textbf{93.72\%} &\cellcolor{tabhighlight}81.50\%  \\

\cline{2-12} 

& \multirow{4}{*}{Grey}
& VACE~\cite{jiang2025vace} &LTX-0.9B  &80  &58.04\%  &95.46\%  &25.00\%  &60.91\%  &\textbf{99.06\%}  &95.68\%  &72.35\%  \\
& & VACE~\cite{jiang2025vace} &Wan-14B  &100  &72.00\%  &97.60\%  &\textbf{30.00\%}  &64.43\%  &98.75\%  &97.45\%  & 76.70\%  \\
& &\cellcolor{tabhighlight}\methodbf (Ours) & \cellcolor{tabhighlight}Wan-14B  & \cellcolor{tabhighlight}4  &\cellcolor{tabhighlight}71.98\% &\cellcolor{tabhighlight}97.66\% &\cellcolor{tabhighlight}\textbf{30.00\%} &\cellcolor{tabhighlight}64.52\% &\cellcolor{tabhighlight}98.81\% &\cellcolor{tabhighlight}97.48\% &\cellcolor{tabhighlight}76.74\%  \\
& &\cellcolor{tabhighlight}\methodppbf (Ours) & \cellcolor{tabhighlight}Wan-14B  & \cellcolor{tabhighlight}4  
&\cellcolor{tabhighlight}\textbf{72.35\%} 
&\cellcolor{tabhighlight}\textbf{97.71\%} 
&\cellcolor{tabhighlight}\textbf{30.00\%}
&\cellcolor{tabhighlight}\textbf{64.62\%}
&\cellcolor{tabhighlight}98.95\% &\cellcolor{tabhighlight}\textbf{97.73\%} &\cellcolor{tabhighlight}\textbf{76.89\%}  \\

\bottomrule
\end{tabular}
\label{tab:comp2}
\vspace{10pt}
\end{table*}

\begin{table*}[t]
\caption{
\textbf{Quantitative evaluations on VACE-Benchmark}. We compare the automated score metrics of our \methodpp with VACE on the dimensions of video quality and video consistency.
}

\centering
\footnotesize
\setlength{\tabcolsep}{6.5pt}
\begin{tabular}{l|l|c|c|cccccc|c}

\toprule

Type & Method & Base Model & \#NFE$\downarrow$ & \multicolumn{7}{c}{\textbf{Video Quality \& Video Consistency}} \\

& & & & \rotatebox{90}{\begin{minipage}{1.8cm} {\tikz\fill[video_quality] (0,0) circle (.5ex);} \centering Aesthetic Quality\end{minipage}}
& \rotatebox{90}{\begin{minipage}{1.8cm} {\tikz\fill[video_quality] (0,0) circle (.5ex);} \centering Background Consistency\end{minipage}}
& \rotatebox{90}{\begin{minipage}{1.8cm} {\tikz\fill[video_quality] (0,0) circle (.5ex);} \centering Dynamic Degree\end{minipage}}
& \rotatebox{90}{\begin{minipage}{1.8cm} {\tikz\fill[video_quality] (0,0) circle (.5ex);} \centering Imaging Quality\end{minipage}}
& \rotatebox{90}{\begin{minipage}{1.8cm} {\tikz\fill[video_quality] (0,0) circle (.5ex);} \centering Motion Smoothness\end{minipage}}
& \rotatebox{90}{\begin{minipage}{1.8cm} {\tikz\fill[video_quality] (0,0) circle (.5ex);} \centering Subject Consistency\end{minipage}}
& \rotatebox{90}{\begin{minipage}{1.8cm} {\tikz\fill[score_average] (0,0) circle (.5ex);} \centering Normalized Average\end{minipage}} \\

\midrule 

\multirow{4}{*}{Depth} 
& VACE~\cite{jiang2025vace} &LTX-0.9B  &80  &55.73\%  &96.06\%  &60.00\%  &67.53\%  &\textbf{98.92\%}  &93.90\%  &78.69\%  \\
& VACE~\cite{jiang2025vace} &Wan-14B  &100  &62.34\%  &96.10\%  &\textbf{65.00\%}  &68.85\%  &98.57\%  &94.25\%  &80.85\%  \\
&\cellcolor{tabhighlight}\methodbf (Ours) & \cellcolor{tabhighlight}Wan-14B  &\cellcolor{tabhighlight}4 &\cellcolor{tabhighlight}62.46\% & \cellcolor{tabhighlight}95.81\% & \cellcolor{tabhighlight}\textbf{65.00\%} & \cellcolor{tabhighlight}69.69\% & \cellcolor{tabhighlight}98.39\% & \cellcolor{tabhighlight}\textbf{94.53\%} &\cellcolor{tabhighlight}80.98\%  \\

& \cellcolor{tabhighlight}\methodppbf (Ours)  &\cellcolor{tabhighlight}Wan-14B  &\cellcolor{tabhighlight}4  &\cellcolor{tabhighlight}\textbf{62.81\%}  &\cellcolor{tabhighlight}\textbf{96.29\%}  &\cellcolor{tabhighlight}\textbf{65.00\%} &\cellcolor{tabhighlight}\textbf{70.16\%}  &\cellcolor{tabhighlight}98.45\%  &\cellcolor{tabhighlight}94.40\%  &\cellcolor{tabhighlight}\textbf{81.19\%}   \\

\midrule 

\multirow{4}{*}{Pose} 
& VACE~\cite{jiang2025vace} &LTX-0.9B  &80  &59.97\%  &94.80\%  &\textbf{85.00\%}  &66.85\%  &\textbf{98.93\%}  &\textbf{94.91\%}  &83.41\%  \\
& VACE~\cite{jiang2025vace} &Wan-14B  &100  &66.69\%  &\textbf{95.32\%}  &75.00\%  &68.48\%  &98.56\%  &94.19\%  &83.04\%  \\
&\cellcolor{tabhighlight}\methodbf (Ours) & \cellcolor{tabhighlight}Wan-14B  &\cellcolor{tabhighlight}4  
&\cellcolor{tabhighlight}\textbf{66.71\%} & \cellcolor{tabhighlight}94.99\% & \cellcolor{tabhighlight}80.00\% & \cellcolor{tabhighlight}\textbf{70.40\%} &\cellcolor{tabhighlight}98.45\% &\cellcolor{tabhighlight}94.26\% &\cellcolor{tabhighlight}84.13\%  \\

& \cellcolor{tabhighlight}\methodppbf (Ours)  &\cellcolor{tabhighlight}Wan-14B  &\cellcolor{tabhighlight}4  &\cellcolor{tabhighlight}66.64\%  &\cellcolor{tabhighlight}94.99\%  &\cellcolor{tabhighlight}\textbf{85.00\%} &\cellcolor{tabhighlight}68.27\%  &\cellcolor{tabhighlight}98.26\%  &\cellcolor{tabhighlight}93.57\%  &\cellcolor{tabhighlight}\textbf{84.45\%}   \\

\midrule 

\multirow{4}{*}{Flow} 
& VACE~\cite{jiang2025vace} &LTX-0.9B  &80  &55.36\%  &95.99\%  &\textbf{75.00\%}  &64.30\%  &\textbf{99.07\%}  &94.14\%  &80.64\%  \\
& VACE~\cite{jiang2025vace} &Wan-14B  &100  &60.44\%  &\textbf{96.09\%}  &\textbf{75.00\%}  &66.99\%  &98.48\%  &94.71\%  &81.95\%  \\
&\cellcolor{tabhighlight}\methodbf (Ours) & \cellcolor{tabhighlight}Wan-14B  &\cellcolor{tabhighlight}4  &\cellcolor{tabhighlight}60.24\% & \cellcolor{tabhighlight}95.65\% & \cellcolor{tabhighlight}\textbf{75.00\%} & \cellcolor{tabhighlight}\textbf{70.02\%} & \cellcolor{tabhighlight}98.42\% & \cellcolor{tabhighlight}94.26\% &\cellcolor{tabhighlight}82.27\%  \\

& \cellcolor{tabhighlight}\methodppbf (Ours)  &\cellcolor{tabhighlight}Wan-14B  &\cellcolor{tabhighlight}4  &\cellcolor{tabhighlight}\textbf{60.96\%}  &\cellcolor{tabhighlight}95.97\%  &\cellcolor{tabhighlight}\textbf{75.00\%} &\cellcolor{tabhighlight}68.81\%  &\cellcolor{tabhighlight}98.50\%  &\cellcolor{tabhighlight}\textbf{95.08\%}  &\cellcolor{tabhighlight}\textbf{82.39\%}   \\

\midrule 

\multirow{4}{*}{Scribble} 
& VACE~\cite{jiang2025vace} &LTX-0.9B  &80  &53.81\%  &96.57\%  &45.00\%  &66.05\%  &\textbf{99.21\%}  &95.23\%  &75.98\%  \\
& VACE~\cite{jiang2025vace} &Wan-14B  &100  &\textbf{59.44\%}  &96.49\%  &\textbf{50.00\%}  &67.01\%  &98.46\%  &95.44\%  &77.80\%  \\
&\cellcolor{tabhighlight}\methodbf (Ours) & \cellcolor{tabhighlight}Wan-14B  &\cellcolor{tabhighlight}4  &\cellcolor{tabhighlight}57.99\% & \cellcolor{tabhighlight}\textbf{96.75\%} & \cellcolor{tabhighlight}\textbf{50.00\%} & \cellcolor{tabhighlight}\textbf{69.16\%} & \cellcolor{tabhighlight}98.60\% & \cellcolor{tabhighlight}\textbf{95.79\%} &\cellcolor{tabhighlight}\textbf{78.05\%} \\

& \cellcolor{tabhighlight}\methodppbf (Ours)  &\cellcolor{tabhighlight}Wan-14B  &\cellcolor{tabhighlight}4  &\cellcolor{tabhighlight}58.53\%  &\cellcolor{tabhighlight}96.58\%  &\cellcolor{tabhighlight}\textbf{50.00\%} &\cellcolor{tabhighlight}67.51\%  &\cellcolor{tabhighlight}98.72\%  &\cellcolor{tabhighlight}95.72\%  &\cellcolor{tabhighlight}77.84\%   \\

\midrule 

\multirow{4}{*}{Grey} 
& VACE~\cite{jiang2025vace} &LTX-0.9B  &80  &58.63\%  &\textbf{95.65\%}  &55.00\%  &59.93\%  &\textbf{98.89\%}  &92.39\%  &76.75\%  \\
& VACE~\cite{jiang2025vace} &Wan-14B  &100 &62.87\%  &95.46\%  &60.00\%  &66.35\%  &98.40\%  &92.91\%  &79.33\%  \\
&\cellcolor{tabhighlight}\methodbf (Ours) & \cellcolor{tabhighlight}Wan-14B  &\cellcolor{tabhighlight}4  &\cellcolor{tabhighlight}62.25\% & \cellcolor{tabhighlight}95.64\% & \cellcolor{tabhighlight}\textbf{65.00\%} & \cellcolor{tabhighlight}\textbf{66.74\%} & \cellcolor{tabhighlight}98.40\% & \cellcolor{tabhighlight}92.88\% &\cellcolor{tabhighlight}\textbf{80.15\%}  \\

& \cellcolor{tabhighlight}\methodppbf (Ours)  &\cellcolor{tabhighlight}Wan-14B  &\cellcolor{tabhighlight}4  &\cellcolor{tabhighlight}\textbf{63.10\%}  &\cellcolor{tabhighlight}95.48\%  &\cellcolor{tabhighlight}60.00\% &\cellcolor{tabhighlight}64.92\%  &\cellcolor{tabhighlight}98.45\%  &\cellcolor{tabhighlight}\textbf{93.08\%}  &\cellcolor{tabhighlight}79.17\%   \\

\midrule 

\multirow{4}{*}{Extension} 
& VACE~\cite{jiang2025vace} &LTX-0.9B  &80  &57.82\%  &\textbf{96.06\%}  &35.00\%  &68.15\%  &\textbf{99.34\%}  &\textbf{94.06\%}  &75.07\%  \\
& VACE~\cite{jiang2025vace} &Wan-14B  &100  &\textbf{61.34\%}  &95.05\%  &60.00\%  &69.55\%  &98.75\%  &92.83\%  &79.59\%  \\
&\cellcolor{tabhighlight}\methodbf (Ours) & \cellcolor{tabhighlight}Wan-14B  &\cellcolor{tabhighlight}4  &\cellcolor{tabhighlight}61.25\% & \cellcolor{tabhighlight}95.94\% & \cellcolor{tabhighlight}45.00\% & \cellcolor{tabhighlight}\textbf{70.01\%} & \cellcolor{tabhighlight}99.04\% & \cellcolor{tabhighlight}93.86\% &\cellcolor{tabhighlight}77.52\%  \\

& \cellcolor{tabhighlight}\methodppbf (Ours)  &\cellcolor{tabhighlight}Wan-14B  &\cellcolor{tabhighlight}4  &\cellcolor{tabhighlight}60.54\%  &\cellcolor{tabhighlight}94.44\%  &\cellcolor{tabhighlight}\textbf{70.00\%} &\cellcolor{tabhighlight}68.68\%  &\cellcolor{tabhighlight}98.67\%  &\cellcolor{tabhighlight}92.53\%  &\cellcolor{tabhighlight}\textbf{80.81\%}   \\

\midrule 

\multirow{4}{*}{Outpaint} 
& VACE~\cite{jiang2025vace} &LTX-0.9B  &80  &56.98\%  &96.47\%  &45.00\%  &\textbf{69.39\%}  &\textbf{99.17\%}  &\textbf{95.50\%}  &77.08\%  \\
& VACE~\cite{jiang2025vace} &Wan-14B  &100 &58.53\%  &96.83\%  &40.00\%  &68.55\%  &99.03\%  &95.34\%  &76.38\%  \\
&\cellcolor{tabhighlight}\methodbf (Ours) & \cellcolor{tabhighlight}Wan-14B  &\cellcolor{tabhighlight}4  &\cellcolor{tabhighlight}\textbf{58.86\%} & \cellcolor{tabhighlight}\textbf{96.84\%} & \cellcolor{tabhighlight}50.00\% & \cellcolor{tabhighlight}69.18\% & \cellcolor{tabhighlight}99.00\% & \cellcolor{tabhighlight}95.48\% &\cellcolor{tabhighlight}78.23\%  \\

& \cellcolor{tabhighlight}\methodppbf (Ours)  &\cellcolor{tabhighlight}Wan-14B  &\cellcolor{tabhighlight}4  &\cellcolor{tabhighlight}58.27\%  &\cellcolor{tabhighlight}96.72\%  &\cellcolor{tabhighlight}\textbf{55.00\%} &\cellcolor{tabhighlight}69.04\%  &\cellcolor{tabhighlight}98.88\%  &\cellcolor{tabhighlight}95.25\%  &\cellcolor{tabhighlight}\textbf{78.86\%}   \\

\midrule 

\multirow{4}{*}{R2V} 
& VACE~\cite{jiang2025vace} &LTX-0.9B  &80 &62.57\%  &97.85\%  &17.50\%  &\textbf{72.32\%}  &\textbf{99.47\%}  &97.91\%  &74.60\%  \\
& VACE~\cite{jiang2025vace} &Wan-14B  &100 &65.99\%  &\textbf{98.77\%}  &15.00\%  &71.19\%  &99.38\%  &\textbf{98.67\%}  &74.83\%  \\
&\cellcolor{tabhighlight}\methodbf (Ours) & \cellcolor{tabhighlight}Wan-14B  &\cellcolor{tabhighlight}4  &\cellcolor{tabhighlight}65.47\% & \cellcolor{tabhighlight}98.58\% & \cellcolor{tabhighlight}20.00\% & \cellcolor{tabhighlight}71.94\% & \cellcolor{tabhighlight}99.32\% & \cellcolor{tabhighlight}98.48\% &\cellcolor{tabhighlight}75.63\%  \\

& \cellcolor{tabhighlight}\methodppbf (Ours)  &\cellcolor{tabhighlight}Wan-14B  &\cellcolor{tabhighlight}4  &\cellcolor{tabhighlight}\textbf{66.51\%}  &\cellcolor{tabhighlight}97.83\%  &\cellcolor{tabhighlight}\textbf{37.50\%} &\cellcolor{tabhighlight}70.90\%  &\cellcolor{tabhighlight}99.00\%  &\cellcolor{tabhighlight}97.43\%  &\cellcolor{tabhighlight}\textbf{78.20\%}   \\




\bottomrule
\end{tabular}
\label{tab:comp_vacebench}
\end{table*}

\noindent\textbf{Discriminator architecture and regularization.}
The discriminator head uses $K$ learnable register tokens, which cross-attend
to feature maps from selected layers $\mathcal S$ of $F_\phi$ and are pooled
with the time embedding by a small MLP into one logit. This reuses the
fake-score backbone without a separate encoder. We regularize real logits
with the finite-difference R1 approximation of APT~\cite{lin2025apt}:
\begin{equation}
  \begin{aligned}
  \Delta_\omega
  &=D_\omega(\bm{x}_t^{\real}+\sigma_{\mathrm{r1}}\bm{\delta},t)
    -D_\omega(\bm{x}_t^{\real},t),\\
  \mathcal{L}_{\mathrm{R1}}(\omega)
  &=w_{\mathrm{r1}}\,
    \mathbb{E}_{\bm{x}_t^{\real},\bm{\delta}}\!\left[
      \left(\frac{\Delta_\omega}{\sigma_{\mathrm{r1}}}\right)^2\right],
  \end{aligned}
  \label{eq:r1}
\end{equation}
where $\bm{\delta}\sim\mathcal N(\bm 0,\bm I)$.

\noindent\textbf{Optimization.}
The default generator and both score networks use the 14B architecture.
The real-score teacher remains frozen throughout training. We optimize the
generator, fake-score network, and discriminator head with FSDP for 2,100
steps on four NVIDIA H200 GPUs, with one sample per GPU. We use AdamW
($\beta_1=0$, $\beta_2=0.999$), generator and
critic learning rates of $2\times10^{-6}$ and $4\times10^{-7}$, respectively,
and a critic-update ratio $R=5$.

\noindent\textbf{Transport.}
We use $\epsilon=0.1$, $\tau_s=0.1$, $\tau_t=10.0$,
$\alpha_{\mathrm{OT}}=10.0$, and 60 Sinkhorn iterations at $2\times$ spatial
downsampling. The DMD and OT weights are $w_{\mathrm{DMD}}=0.9$ and
$w_{\mathrm{OT}}=0.1$. For nonnegative measures, we use the generalized KL
divergence
\[
  \mathrm{KL}(\bm{p}\,\|\,\bm{q})
  = \sum_k\left[p_k\log\frac{p_k}{q_k}-p_k+q_k\right],
\]
with the standard continuous extension at zero. Algorithm~\ref{alg:auot_sinkhorn}
gives the complete log-domain solver used to compute the asymmetric transport
plan. It optimizes Eq.~\eqref{eq:auot} under the negative-entropy convention
$\mathcal{E}(\Tplan)=\sum_{i,j}T_{ij}(\log T_{ij}-1)$ and the generalized KL
defined above.

\begin{algorithm}[t]
\SetAlgoLined
\KwIn{Cost matrix $\Cmat\in\mathbb{R}^{I\times J}$; marginals
$\bm{u}\in\mathbb{R}_{+}^{I}$ and $\bm{\mu}\in\mathbb{R}_{+}^{J}$;
$\epsilon,\tau_s,\tau_t>0$; tolerance $\texttt{tol}$; maximum iterations $L$}
\KwOut{Asymmetric unbalanced transport plan
$\Topt\in\mathbb{R}_{+}^{I\times J}$}
$\log\bm{K}\leftarrow-\Cmat/\epsilon$\;
$\rho_s\leftarrow\tau_s/(\tau_s+\epsilon)$,
$\rho_t\leftarrow\tau_t/(\tau_t+\epsilon)$\;
$\bm{f}\leftarrow\bm{0}_{I}$, $\bm{g}\leftarrow\bm{0}_{J}$\;
\For{$\ell=1,\ldots,L$}{
  $\bm{f}_{\mathrm{old}}\leftarrow\bm{f}$,
  $\bm{g}_{\mathrm{old}}\leftarrow\bm{g}$\;
  \For{$i=1,\ldots,I$}{
    $f_i\leftarrow\rho_s\!\left[\log u_i-
    \operatorname{LSE}_{j}\!\left(\log K_{ij}+g_j\right)\right]$\;
  }
  \For{$j=1,\ldots,J$}{
    $g_j\leftarrow\rho_t\!\left[\log\mu_j-
    \operatorname{LSE}_{i}\!\left(\log K_{ij}+f_i\right)\right]$\;
  }
  \If{$\|\bm{f}-\bm{f}_{\mathrm{old}}\|_1+
  \|\bm{g}-\bm{g}_{\mathrm{old}}\|_1<\texttt{tol}$}{
    \textbf{break}\;
  }
}
$T^*_{ij}\leftarrow\exp\!\left(f_i+\log K_{ij}+g_j\right)$
for all $(i,j)$\;
\KwRet{$\Topt$}\;
\caption{Log-domain generalized Sinkhorn algorithm for asymmetric
unbalanced OT in Eq.~\eqref{eq:auot}. Here
$\operatorname{LSE}_{k}(z_k)=\log\sum_k\exp(z_k)$.}
\label{alg:auot_sinkhorn}
\end{algorithm}

\noindent\textbf{Adversarial training.}
The head uses $K=4$ register tokens of dimension $d_r=512$, cross-attending
to layers 15, 25, 35, and 39 of the fake-score network. For Wan2.1-14B, the
backbone hidden dimension is 5120. We use
$w_{\mathrm{GAN-G}}=0.05$, $w_{\mathrm{GAN-D}}=0.01$, critic timestep shift
12, $w_{\mathrm{r1}}=30$, and $\sigma_{\mathrm{r1}}=0.1$. Adversarial
training begins after 900 warm-up steps; GAN-G and GAN-D are applied every
fifth generator step and every tenth critic step, respectively. These settings
are shared across T2V, I2V, and VACE. For cross-scale distillation, we scale
the discriminator learning rate by 0.2, set $w_{\mathrm{r1}}=60$, and reduce
$w_{\mathrm{GAN-G}}$ to 0.02.

\section{Multi-Task Data and UVCBench Construction}
\label{sec:supp_data}

We follow the automated data and benchmark construction pipeline introduced
in our conference version~\cite{wang2026vdot}; the details required to
reproduce the training corpus and \methodbench\ are summarized below.

\noindent\textbf{Training data.}
We collect 250\,000 4K Artgrid videos, rescale them to
$832\times480$, and caption them with InternVL~\cite{chen2024internvl}.
Captions serve as both prompts and filtering criteria. We express the first
frame for I2V and automatically extracted depth, pose, optical flow, scribble,
and mask signals for VACE-style tasks in the shared condition format. For
pose, for example,
Qwen3~\cite{yang2025qwen3} filters captions without human subjects; extracted
pose sequences are ranked by weighted inter-frame keypoint displacement, and
the highest-ranked samples are retained. Analogous curation is applied to
each condition type.

\noindent\textbf{Benchmark construction.}
\methodbench\ contains 8 single-condition and 10 composite tasks, with 20
videos per task. Single-condition tasks cover pose, depth, optical flow,
greyscale, scribble, face/object reference, canvas outpainting, and temporal
extension. Composite tasks pair a reference image or first frame with pose,
depth, flow, greyscale, or scribble. Qwen3-Max expands task descriptions into
candidate prompts; Wan2.1-14B synthesizes V2V/MV2V exemplars, while VACE-14B
uses Qwen-Image~\cite{wu2025qwenimagetechnicalreport} references for R2V and
composite tasks. The corresponding inputs and masks are derived automatically,
then manually verified and regenerated when necessary.

\noindent\textbf{Evaluation protocol.}
Generated videos on \methodbench\ are scored using six VBench dimensions:
Aesthetic Quality, Background Consistency, Dynamic Degree, Imaging Quality,
Motion Smoothness, and Subject Consistency. We additionally report their
normalized average as a compact overall measure while retaining the
individual dimensions to expose quality--motion trade-offs.

\section{Full Results on the UVCBench Benchmark}
\label{sec:supp_uvcbench}

Table~\ref{tab:uvcbench_single_supp} complements the representative
single-condition results in Sec.~\ref{sec:main_results} with depth, scribble,
and temporal extension. On the densely specified depth and scribble tasks,
\methodpp\ has slightly lower normalized averages than \method; on extension,
it preserves a comparable overall score while improving imaging and subject
consistency. Table~\ref{tab:comp2} further reports the full comparison on all
ten composite tasks of UVCBench~\cite{wang2026vdot}. Relative to \method,
\methodpp\ obtains higher normalized averages on some tasks and lower ones
on others, with lower imaging quality on most composite tasks.
FirstFrame+Flow illustrates the appearance--dynamics trade-off, combining
higher Dynamic Degree with lower Imaging Quality.

\section{Full Results on the VACE Benchmark}
\label{sec:supp_vacebench}

Table~\ref{tab:comp_vacebench} reports the full per-task comparison on the
VACE benchmark~\cite{jiang2025vace}, complementing the \methodbench\
results in Sec.~\ref{sec:main_results}. Despite using only four NFEs,
\methodpp\ remains competitive with the many-step VACE variants across the
densely constrained depth, pose, flow, scribble, and greyscale tasks. The
differences become more pronounced on extension, outpainting, and R2V, where
the inputs specify only part of the target video. On these tasks, the stronger
normalized averages are driven mainly by increased dynamics while the
consistency metrics remain high. This pattern supports the intended role of
the relaxed student marginal: uncertain regions need not be forced into
arbitrary matches, allowing plausible content and motion to develop around the
observed evidence.

\end{document}